\documentclass{ecai}

\usepackage{times}
\usepackage{graphicx}
\usepackage{latexsym}

\usepackage[utf8]{inputenc}

\usepackage{natbib}
\usepackage{hyperref}
\usepackage[normalem]{ulem}
\usepackage{url}
\usepackage[T1]{fontenc}

\usepackage{tabularx}
\usepackage{amsfonts}
\usepackage{booktabs}
\usepackage{amssymb}
\usepackage{amsthm}
\usepackage{dsfont}
\usepackage{stmaryrd}
\usepackage{graphicx}
\usepackage{tikz}
\usepackage{subfigure}
\usepackage{algorithm}
\usepackage{amsmath}
\usepackage[normalem]{ulem}
\usepackage{booktabs}
\usepackage{multirow}

\def\Z{\mathbb{Z}}
\def\R{\mathbb{R}}
\def\Ex{\mathbb{E}}
\def\I{\mathbb{I}}

\def\nor{\mathrm{N}}
\def\cov{\mathrm{cov}}
\def\mean{\mathrm{mean}}

\def\mean{\mathrm{mean}}
\def\recerror{\mathrm{rec\_error}}

\def\D{\mathcal{D}}
\def\E{\mathcal{E}}
\def\Z{\mathcal{Z}}

\def\X{\mathbf{X}}

\def\wica{WICA}

\def\ii{\textbf{wii}}

\def\D{\mathcal{D}}

\newtheorem{observation}{Observation}[section]
\newtheorem{theor}{Theorem}[section]

\def\wica{\mbox{WICA}}

\def\mixname{\mbox{OTS}}

\theoremstyle{definition}
\newtheorem{definition}{Definition}[section]

\begin{document}

\title{\wica{}: nonlinear weighted ICA}

\author{Andrzej Bedychaj \institute{Jagiellonian University,
Poland, email: andrzej.bedychaj@gmail.com} \and Przemysław Spurek \and Aleksandra Nowak \and Jacek Tabor }

\maketitle

\begin{abstract}
Independent Component Analysis (ICA) aims to find a coordinate system in which the components of the data are independent. 
In~this paper we construct a new nonlinear ICA model, called \wica{}, which 
obtains better and more stable results than other algorithms.
A~crucial tool is given by a new efficient method of verifying nonlinear dependence with the use of computation of correlation coefficients for normally weighted data. In addition, authors propose a new baseline nonlinear mixing to perform comparable experiments, and a~reliable measure which allows fair comparison of nonlinear models.
Our code for \wica{} is available on Github\footnote[2]{\url{https://github.com/gmum/wica}}.
\end{abstract}

\section{Introduction}

The goal of Linear Independent Component Analysis (ICA) is to find such a unmixing function of the given data that the resulting representation has statistically independent components. Common tools solving this problem are based on maximizing some measure of nongaussianity, e.g. kurtosis \Citep{hyvarinen1999fast,bell1995information} or skewness~\Citep{spurek2017ica}.
Clearly, an obvious limitation of those approaches is the assumption of linearity,  as the real world data usually contains complicated and nonlinear dependencies (see for instance
\Citep{larson1998radio,ziehe2000artifact}). 
Designing an efficient and easily implementable nonlinear analogue~of~ICA~is~a~much~more complex problem than its linear counterpart. A crucial complication is that without any limitations imposed on the space of the mixing functions the problem of nonlinear-ICA is ill-posed, as there are 
infinitely many valid solutions \Citep{hyvarinen1999nonlinear}.

%Linear Independent Components Analysis (ICA) is an important data analysis technique, which aims to identify a \emph{linear} function such that the components of the dataset obtained after the transformation are~independent. Commonly, the independence is approximated using some measure of nongaussianity, e.g. kurtosis \Citep{hyvarinen1999fast,bell1995information} or skewness~\Citep{spurek2017ica}.
%Clearly, the obvious limitation of those approaches to the ICA is a restriction to the linear transformations, as the real world data usually contains complicated and nonlinear dependencies, see for instance
%\Citep{larson1998radio,ziehe2000artifact}. 
%Designing an efficient and easily implementable nonlinear analogue~of~ICA~is~a~much~more complex problem than its linear counterpart. A fundamental complication is that the solution of nonlinear~ICA~is~in principle non-identifiable, as without any constraints on the space of the mixing functions, there exists an~infinite number of valid solutions \Citep{hyvarinen1999nonlinear}. 

As an alternative to the fully unsupervised setting of the nonlinear~ICA one can assume some prior knowledge about the distribution of the sources, which allows to obtain identifiability \Citep{hyvarinen2016unsupervised,hyvarinen2019nonlinear}. Several algorithms exploiting this property have been recently proposed, either assuming access to segment labels of the sources \Citep{hyvarinen2016unsupervised}, temporal dependency of the sources \Citep{hyvarinen2017nonlinear} or, generally, that the sources are conditionally independent, and the conditional variable is observed along with the mixes \Citep{hyvarinen2019nonlinear, khemakhem2019variational}. However it may be sometimes hard to generalize those approaches in fully unsupervised setting where some prior knowledge is unavailable or the qualities of the data itself preserve unknown for the researcher.

An additional complication in devising nonliear-ICA algorithms lies in proposing an efficient measure of~independence, which optimization would encourage the model to disentangle the components. One of the most common nonlinear method is MISEP \Citep{almeida2003misep} which, similar to the popular INFOMAX algorithm \Citep{bell1995information}, uses the mutual information criterion. In consequence, the procedure involves the~calculation of the Jacobian of the modeled nonlinear transformation, which often causes a computational overhead when both the input and output dimensions are large.

Another approach is applied in~NICE (Nonlinear Independent Component Estimation) \Citep{dinh2014nice}. Authors propose a fully invertible neural network architecture where the Jacobian is trivially obtained. The independent components are then estimated using the maximum likelihood criterion. The drawback of both MISEP and NICE is that they require choosing the prior distribution family of the unknown independent components. An alternative approach is given by ANICA (Adversarial nonlinear ICA) \Citep{brakel2017learning}, where the independence measure is directly learned in each task with the use of GAN-like adversarial method combined with an an~autoencoder architecture. However, the introduction of a GAN-based independence measure results in an often unstable adversarial training. 

In this paper we present a competitive approach to nonlinear independent components analysis~--~\wica{} (Nonlinear Weighted~ICA). Crucial role in our approach is played by the conclusion from~\Citep{bedychaj2019independent}, which proves that to verify nonlinear independence it is sufficient to~check the linear independence of the normally weighted dataset, see Fig. \ref{fig:cov}. Based on this result we introduce {\em weighted indepedence index ($\ii{}$)} which relies on computing weighted covariance and~can be applied to the verification of the nonlinear independence, see Section~\ref{section:wc}.
Consequently, the constructed \wica{} algorithm is based on simple operations on matrices, and therefore~is~ideal~for~GPU~calculation and parallel processing. We construct it by incorporating the introduced cost function in a commonly used in ICA problems auto-encoder framework \Citep{brakel2017learning,le2011ica}, where the~role of the decoder is to limit the unmixing function so that the~learned by the encoder independent components contained the~information needed to reconstruct the inputs, see Section~\ref{sec:algorithm}.

We verified our algorithm in the case of a source signal separation problem. In Section \ref{sec:experiments}, we presented the results of \wica{} for nonlinear mixes of images and for the decomposition of electroencephalogram signals.  
It occurs that \wica{} outperforms other methods of nonlinear ICA, both with respect to unmixing quality and the stability of the results, see Fig.~\ref{fig:spearman_1}. 

To fairly evaluate various nonlinear ICA methods in the case of~higher dimensional datasets, we introduce a measure index called \mixname{} based on Spearman's rank correlation coefficient. In the definition of~\mixname{}, similarly to the clustering accuracy (ACC) \Citep{cai2005document,cai2010locally}, we~used optimal transport to obtain the minimal mismatch cost. This approach has its merit here, since the~correspondence between the~input coordinates and the reconstructed components in a higher dimensional space is nontrivial. 

Another important ingredient of this paper is the introduction of~a~new and fully invertible nonlinear mixing function. In the case of linear ICA, one can easily construct many experiment settings that can be used in order to evaluate and compare different methods. Such standards are unfortunately not present in the case of nonlinear~ICA. Therefore it is not clear what kind of nonlinear mixing should be used in the benchmark experiments. In most cases the authors usually use mixing functions, which correspond with the models architecture \Citep{almeida2003misep,brakel2017learning}. In contrast to such methodology, we propose a new iterative nonlinear mixing function based on the flow models \Citep{dinh2014nice, glow}. This method does not relates to internal design of our network architecture, is invertible and allows for chaining the task complexity by varying the number of iterations, making it a useful tool in verification of the nonlinear ICA models.

%%%%%%%%%%%%%%%%%%%%%%%%%%%%%%%%%%%%%%%%%%%%%%%%%%%%%%%%%%%%%%%%%%

\section{Weighted independence index}
\label{section:wc}

Let us consider a random vector $\X$ in $\R^d$ with density $f$. Then $\X$ has independent components iff $f$ factors as 
$$
f\left(x_1,x_2, \ldots,x_d\right)=f_1(x_1) \cdot f_2(x_2) \cdot \ldots \cdot f_d(x_d),
$$
for some densities $f_i$, where $i\in \{1,2,\ldots,d\}$. Those functions are called marginal densities of $f$. A related, but much weaker notion, is the uncorrelatedness.  We say that $\X$ has uncorrelated components, if the covariance of $\X$ is diagonal. Contrary to the independence, correlation has fast and easy to compute estimators. Components independence implies uncorrelatedness, but the opposite is not valid, see Fig.~\ref{fig:cov}.  

\begin{figure}[!h]
\centering
\includegraphics[width=0.49\linewidth]{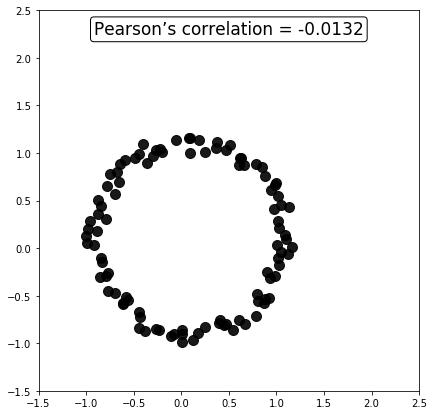}
\includegraphics[width=0.49\linewidth]{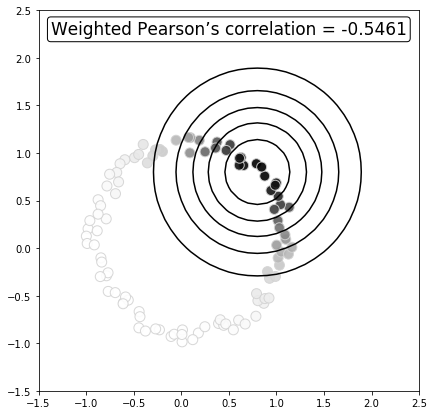}
\caption{Sample from a random vector which Pearson’s correlation is equal to zero (left), but the components are not independent. Since the components are not independent, one can choose Gaussian weights so that Pearson’s correlation of weighted dataset~is~not~zero~(right).}
\label{fig:cov}
\end{figure}

Let us mention that there exist several measures which verify the~independence. One of the most well-known measures of independence of random vectors is the {\em distance correlation} (dCor) \Citep{szekely2007measuring}, which is applied in \Citep{matteson2017independent} to solve the linear ICA problem. Unfortunately, to verify the independence of components of the~samples, dCor~needs~$2^d N^4$~comparisons, where $d$ is the dimension of the sample and $N$ is the sample size. Moreover, even a simplified version of dCor which checks only pairwise independence has high complexity and does not obtain very good results (which can be seen in experiments from Section \ref{sec:experiments}). This motivates the research into fast, stable and efficient measures of indepedence, which are adapted to GPU processing.

\subsection{Introducing the \ii{} index}

In this subsection we fill this gap and introduce a method of verifying independence which is based on the covariance of the weighted data. The covariance scales well with respect to the sample size and data dimension, therefore the proposed covariance-based index inherits similar properties.

To proceed further, let us introduce weighted random vectors.
\begin{definition}
Let $w:\R^d \to \R_+$ be a bounded weighting function. By $\X_w$ we denote a weighted random vector with a density\footnote[3]{This is just the normalization of $w(x) f(x)$.}
$$
f_w(x)= \frac{w(x)f(x)}{\int w(z)f(z)dz}.
$$
\end{definition}

\begin{observation} \label{pr:1}
Let $\X$ be a random vector which has independent components, and let $w$ be an arbitrary weighting function. Then $\X_{w}$ has independent components as well.
\end{observation}

One of the main results of \Citep{bedychaj2019independent} is that the strong version of the~inverse of the above theorem holds. Given $m \in \R^d$, we consider the~weighting of $\X$ by the standard normal gaussian with center at $m$ ($\nor(m,\I)$):
$$
\X_{[m]}=\X_{\nor(m,\I)}.
$$
We quote the following result which follows directly from the proof of Theorem 2 from \Citep{bedychaj2019independent}:

\begin{theor}
\label{theorem:1}
Let $\X$ be a random vector, let $p \in \R^d$ and $r>0$ be arbitrary. If $\X_{[q]}$ has linearly independent components for every $q \in B( p,r)$, where $B(p,r)$ is a ball with center in $p$ and radius $r$, then $\X$ has the~independent components.
\end{theor}

Given sample $X=(x_i) \subset \R^d$, vector $p \in \R^d$, and weights $w_i = \nor(p,\I)(x_i)$, we define the weighted sample as:
$$
X_{[p]}=(x_i,w_i).
$$
Then the mean and covariance for the weighted sample \mbox{$X_{[p]}=(x_i,w_i)$} is given by:
$$
\mean X_{w}=\frac{1}{\sum_i w_i}\sum_i w_i x_i
$$
and 
$$
\cov X_{w}=\frac{1}{\sum_i w_i}\sum_i w_i (x_i-\mean X_w)^T(x_i-\mean X_w).
$$

The informal conclusion from the above theorem can be stated as follows: \textit{if 
$\cov X_{[p]}$ is (approximately) diagonal for a sufficiently large set of $p$, then the sample $X$ was generated from a distribution with independent components.}

Let us now define an index which will measure the distance from being independent. We define the {\em weighted independence index} $\left(\ii(X,p)\right)$ as 
$$
\ii(X,p)=\frac{2}{d(d-1)} \sum_{i<j} c_{ij},
$$
where $d$ is the dimension of $X$ and
$$
c_{ij}=\frac{2z_{ij}^2}{z_{ii}^2+z_{jj}^2},
$$
for $Z=[z_{ij}]=\cov X_{[p]}.$

\begin{observation}

Let us first observe that $c_{ij}$ is a close measure to the correlation $\rho_{ij}$, namely:
$$
c_{ij} \leq \rho^2_{ij},
$$
where the equality holds iff the $i$-th and $j$-th components in $X_{[p]}$ have equal standard deviations.
\end{observation}

\begin{proof}
Obviously 
$$
\rho^2_{ij}=\frac{z_{ij}^2}{z_{ii} \cdot z_{jj}}.
$$
Since $ab \leq \frac{1}{2}(a^2+b^2)$ (where the equality holds iff $a=b$), we obtain the assertion of the observation.
\end{proof}

Consequently, $\ii(X,p)=1$ iff all components of $X_{[p]}$ are linearly dependent and have equal standard deviations. Thus, the minimization of $\ii$ simultaneously aims at maximizing the independence and increasing the difference between the standard deviations.

We extend the index for a sequence of points $\{p_1,p_2,\ldots,p_n\}$ , as the mean of the indexes for each $p_i$:
$$
\ii(X;\{p_1,p_2,\ldots,p_n\})=\frac{1}{n}\sum_{i=1}^n\ii(X,p_i).
$$

%%%%%%%%%%%%%%%%%%%%%%%%%%%%%%%%%%%%%%%%%%%%%%%%%%%%%
\subsection{Selecting the weighting points}
%\section{Description of the \wica{} algorithm}
\label{sec:weightpoints}

To implement the weighted independence index in practice, we need to find the optimal choice of weighting centers $(p_i)$. % \ola{$(p_i)$}. 
First, we assume that the dataset in question is normalized componentwise (in particular, variance of each coordinate is one).
We argue that the 
right choice of $(p_i)$ should satisfy the following two conditions:
\begin{itemize}
    \item selected weights do not concentrate on a small percentage of the~data,
    \item for different centers selected from the dataset, weights diversify the data points.
\end{itemize}
At first glance, it would seem that the simplest choice for points $(p_i)$ is to sample them from the standard normal distribution.  However, the conducted by us preliminary experiments (see Fig. \ref{fig:co}) demonstrate that sampling from 
%\ola{COS JEST NIE TAK Z LABELAMI DO FIGUR!} 
$\nor\left(0,\tfrac{1}{d}\I\right)$ would be a better choice.

% \begin{wrapfigure}{r}{0.5\textwidth}
% \centering
% \includegraphics[width=0.5\textwidth]{img/p_1.png}
% \caption{In the experiment, we sampled twenty points from $\nor(0,\I)$~\mbox{(x-axis)}. Then, we calculate weights of the points respectively to $\nor(0,\I)$ and $\nor\left(0,\tfrac{1}{d}\I\right)$. We present values of those weights (sorted decreasingly) in the case when the center is chosen according to \mbox{$\nor(0,\I)$ vs. $\nor\left(0,\tfrac{1}{d}\I\right)$}. One can see that weights derived from $\nor\left(0,\tfrac{1}{d}\I\right)$ actually balance more data points, in contrary to $\nor(0,\I)$ which focus on smaller amount of data ($\nor(0,\I)$ converges to~0~earlier). }\label{fig:co}
% \end{wrapfigure}

\begin{figure}
%\begin{wrapfigure}{r}{0.5\textwidth}
\centering
\includegraphics[width=0.7\linewidth]{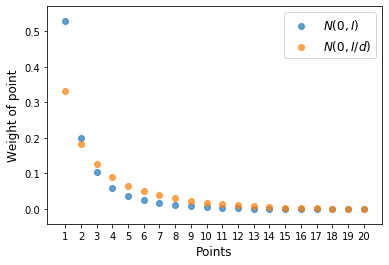}
\caption{In the experiment, we sampled twenty points from $\nor(0,\I)$~\mbox{(x-axis)}. Then, we calculate weights of the points respectively to $\nor(0,\I)$ and $\nor\left(0,\tfrac{1}{d}\I\right)$. We present values of those weights (sorted decreasingly) in the case when the center is chosen according to \mbox{$\nor(0,\I)$ vs. $\nor\left(0,\tfrac{1}{d}\I\right)$}. One can see that weights derived from $\nor\left(0,\tfrac{1}{d}\I\right)$ actually balance more data points, in contrary to $\nor(0,\I)$ which focus on smaller amount of data ($\nor(0,\I)$ converges to~0~earlier). }\label{fig:co}
%\end{wrapfigure}
\end{figure}

%\subsection{Theoretical analysis.}
%Let us now discuss the theoretical foundations behind the results from Fig. \ref{fig:co}. 
Consider the case when the data come from the standard normal distribution. For given weights $w$ and density $f$ we define measure $P(w,f)$ as:

\begin{equation}
P(w,f)=\frac{\left(\int w(x) f(x) dx\right)^2}{\int w^2(x) f(x) dx}. \label{eq:P_measure}
\end{equation}

Observe that if $w$ is constant on a subset $U$ of some space $S$ (for which functions $w$ and $f$ are well-defined) and zero otherwise, then the above reduces to $\mu(U)$, where $\mu$ is counting measure. Intuitively, \textit{$P(w,f)$ returns the percentage of the population which has nontrivial weights.}

Let us consider the case when $\mu$ is given by the standard normal density
$$
w(x)=\nor(p,\I)(x)
$$
and our dataset is normalized as stated above. Then, directly from~(\ref{eq:P_measure}), one obtains:
$$
P(w,f)=\frac{\left(\int \nor(p,\I)(x) \nor(0,\I)(x)dx\right)^2}{\int \nor(p,\I)^2(x)\nor(0,\I)(x)dx }
$$
Applying the formula for the product of two normal densities:
$$
\nor(m_1,\Sigma_1)(x) \cdot \nor(m_2,\Sigma_2)(x) \, =c_c\nor(m_c,\Sigma_c)(x),
$$
where 
$$ c_c=\nor(m_1-m_2,\Sigma_1+\Sigma_2)(0), $$
$$\Sigma_c=(\Sigma_1^{-1}+\Sigma_2^{-1})^{-1},$$
and 
$$m_c=\Sigma_c(\Sigma_1^{-1}m_1+\Sigma_2^{-1} m_2),$$
we get:
$$
\int \nor(p,\I)(x) \nor(0,\I)(x)dx=\nor(p,2\I)(0),
$$
for the numerator, and 
$$
\int \nor\left(p,\I\right)^2(x)\nor(0,\I)(x)dx=\nor(0,2\I)(0) \nor\left(p,\tfrac{3}{2}\I\right)(0).
$$
for the denominator. The equation for the denominator follows from the simple fact that:
$$
\nor(p,\I)^2(x)=\nor(0,2\I)(0) \cdot \nor\left(p,\tfrac{1}{2}\I\right)(x),
$$
Summarizing, we obtain that 
\begin{equation} \label{eq:eq2}
\begin{split}
P(\nor(p,\I),\nor(0,\I)) & = \frac{\nor(p,2\I)^2(0)}{\nor(0,2\I)(0) \nor\left(p,\tfrac{3}{2}\I\right)(0)} \\
 & = \left(\tfrac{3}{4}\right)^{D/2} \exp\left(-\tfrac{1}{6}\|p\|^2\right).
\end{split}
\end{equation}
% $$
% P(\nor(p,I),\nor(0,I))=\frac{\nor(p,2I)^2(0)}{\nor(0,2I)(0) \nor(p,\tfrac{3}{2}I)(0)}
% $$
% $$
% =(\tfrac{3}{4})^{D/2} \exp(-\tfrac{1}{6}\|p\|^2).
% $$
Normalizing (\ref{eq:eq2}) by its maximum obtained at $0$, we get
$$
\exp\left(-\tfrac{1}{6}\|p\|^2\right).
$$

Clearly if $p$ would be chosen from the standard normal distribution, the value of $\|p\|^2$
for large dimensions equals approximately $d$, and consequently the weights for the randomly chosen points will become concentrated at a single point (see Fig \ref{fig:co}).
To obtain the quotient approximately constant, we should choose $p$ so that its norm is approximately one. Hence, it leads to the choice of $p$ from the distribution $\nor\left(0,\tfrac{1}{d}\I\right)$.

One can observe, that if $\X \sim \nor(0,\I)$, then we can sample from $\nor\left(0,\tfrac{1}{d}\I\right)$ by taking the mean of $d$ randomly chosen vectors from $\X$. This leads to the following definition:

\begin{definition}
For the dataset $X \subset \R^d$, we define
$$
\ii(X)=\Ex \{\ii(Y,p): \mbox{$p$ a mean of random $d$ elements of $Y$}\},
$$
where $Y$ is a componentwise normalization of $X$ and $\Ex$ stands for expected value.
\end{definition}

Let us summarize why centering the weights at the mean of $d$ elements from the dataset has good properties:
\begin{itemize}
    \item if the data is restricted to some subspace $S$ of the space, then mean also belongs to $S$;
    \item if the data comes from normal distribution $\nor(m,\Sigma)$, then mean of $d$ elements comes from $\nor\left(m,\tfrac{1}{d}\Sigma\right)$,
    \item if the data has heavy tails (i.e. comes from Cauchy distribution), then~the~distribution of mean for $d$ elements set can be close to the original dataset mean.
\end{itemize}

\section{The WICA algorithm}
\label{sec:algorithm}

In this section we propose the \wica{} algorithm for nonlinear ICA decomposition which exploits the $\ii(X)$ index in practice.

Following \Citep{brakel2017learning}, we use an auto-encoder (AE) architecture, which consists of an encoder function $\E:~\R^d~\to~\Z$ and a complementary decoder function $\D:\Z \to \R^d$. The role of the encoder is to learn a transformation of the data that unmixes the latent components, utilizing some measure of independence (we use the $\ii(X)$ index). The decoder is responsible for limiting the encoder, so that the learned representation does not lose any information about the input. In practice, this is implemented by simultaneously minimizing the reconstruction error:

%Standard auto-encoders aim to enforce a coding of the input variables that minimizes the reconstruction error:
$$\recerror(X;\E,\D)=\sum_{i=1}^d \|x_i-\D(\E x_i)\|^2.$$

%Minimizing the difference between the input and the output is crucial to recover unmixing mapping close to inverse of the mixing one. The funal cost funtion 

%The realisation of the $\ii(X)$ index in practice is provided by our algorithm~--~\wica{}~--~a~nonlinear ICA model based on the $\ii$ index. Following the ANICA \Citep{brakel2017learning}, we used an auto-encoder (AE) architecture.

%\subsection{Auto-encoder model.} Let $X \subset \R^d$ denote the input data. An~auto-encoder is a model consisting of an encoder function $\E:~\R^d~\to~\Z$ and a complementary decoder function $\D:\Z \to \R^d$, aiming to enforce coding of the input variables that minimizes the reconstruction error:
%$$\recerror(X;\E,\D)=\sum_{i=1}^d \|x_i-\D(\E x_i)\|^2.$$

%\subsection{\wica{} cost function.} In this paragraph we describe the actual cost function used in the \wica{} model.

%To obtain independence in the latent space, we add to the standard auto-encoder cost function the weighted independence index $\ii$~computed in the latent.
Reducing the difference between the input and the output is crucial to recover unmixing mapping close to inverse of the mixing one. Thus our final cost function is given by
\begin{equation}
\mathrm{cost}(X;\E,\D)=\recerror(X;\E,\D)+\beta \ii(\E X). \label{eq:cost_function}
\end{equation}
where $\beta$ is a hyperparameter which aims to weight the role of reconstruction with that of independence (analogous~to~\mbox{$\beta$-VAE}~\Citep{Higgins2017betaVAELB}). The~training procedure follows the steps:

\begin{algorithm}[H] \caption{\wica{}}\label{alg:algorithm_wica}
% For a mini-batch $X'$ from the dataset $X$ we proceed by applying the following procedure:
\begin{enumerate}
\item Take mini-batch $X'$ from the dataset $X$.
\item Normalize componentwise $\E X'$, to obtain $Y$
\item Compute $p_1,\ldots,p_d$, where $p_i$ is the mean of randomly chosen~$d$~elements from $Y$,
\item Minimize:
$$
\recerror(X';\E,\D) + \beta \ii(Y;p_1,\ldots,p_d).
$$
\end{enumerate}
\end{algorithm}

%\section{Results} 

%In this section we introduce a new nonlinear mixing setup, which is robust and provide comparable level of complexity for each of the tested models. Secondly, we propose an OTS measure - the metric that allows disentanglements from ICA algorithms to be compared on a fair ground. In the last subsection we provide experiments to gain intuition of \wica{} results in the comparison to other methods. 

\section{Nonlinear mixing} \label{sec:mixing}
%\andrzej{Czy nie lepiej zebrać trzy kolejne roździały w jeden np. "Results" i podzielić to na subsekcje? Jaki jest sens 3 osobnych rozdziałów, skoro naszymi rezultatami jest (I) mixing function, (II) miara OTS (III) WICA?}

%\ola{INTRO}
Let us start with a discussion of possible definitions of the nonlinear mixing function used for benchmarking the ICA methods. In the beginning we shortly explain some approaches used in the linear ICA, and then move forward to propose a mixing which benefits from properties desired in the comparison of the results obtained by nonlinear ICA algorithms.

In the case of linear ICA the experiments are usually conveyed on an artificial dataset, which is obtained by mixing two or more of~independent source signals. This allows for the comparison of~the~results returned by the analyzed methods with the original independent components. In the real-world applications such a procedure is of~course infeasible, but in experimental setting it provides a good basis for~benchmarking different models. In classical ICA setup, creating an artificial mixing function is equivalent to selecting a random invertible matrix $A$, such that $X=A\cdot S$, where $S$ are the true sources and $X$ are the observations, which are then passed to the evaluated methods. Such mixing is used by \Citep{bedychaj2019independent,hyvarinen1999fast,spurek2017ica}.

Unfortunately, there do not exist any mixing standards for the nonlinear ICA problem. A common setup of the comparable environments needed to test the nonlinear models of ICA is to interlace linear mixes of signals with nonlinear functions \Citep{almeida2003misep, brakel2017learning}. During our experiments we found that the proposed methods of nonlinear mixes are ineffective in large dimensions. The aforementioned approaches usually apply only a shallow stack of linear projections followed by a nonlinearity. In consequence, the obtained observations are either close to the linear mixing (and therefore not hard enough to be properly challenging for the linear models) or become degenerate (i.e.~all~points cluster towards zero). Results of such mixing techniques are presented on Fig. \ref{lattice_anica}.

% \begin{figure}
% \centering
% %\subfigure[Iteration 0]{
% %\includegraphics[width=0.3\linewidth]{anica_mieszanie_krata/ANICA_2D_pnl_0.png}
% %}
% %\\
% \subfigure[PNL - Iteration 1]{
% \includegraphics[width=0.44\linewidth]{anica_mieszanie_krata/ANICA_2D_pnl_1.png}
% }
% \subfigure[MLP - Iteration 1]{
% \includegraphics[width=0.44\linewidth]{anica_mieszanie_krata/ANICA_2D_mlp_1.png}
% }
% \\
% %\subfigure[PNL - Iteration 2]{
% %\includegraphics[width=0.3\linewidth]{anica_mieszanie_krata/ANICA_%2D_pnl_2.png}
% %}
% %\subfigure[MLP - Iteration 2]{
% %\includegraphics[width=0.3\linewidth]{anica_mieszanie_krata/ANICA_%2D_mlp_2.png}
% %}
% \\
% \subfigure[PNL - Iteration 3]{
% \includegraphics[width=0.44\linewidth]{anica_mieszanie_krata/ANICA_2D_pnl_3.png}
% }
% \subfigure[MLP - Iteration 3]{
% \includegraphics[width=0.44\linewidth]{anica_mieszanie_krata/ANICA_2D_mlp_3.png}
% }
% \caption{Results of the nonlinear mixing techniques proposed in \Citep{brakel2017learning} on a synthetic lattice data. Post nonlinear mixing model~(PNL) introduced only slight nonlinearities, which are not hard enough to solve even for the linear algorithms. On the other hand, the multi-layer perceptron mixing (MLP) technique collapses after just couple of iterations.} %In such a regime, only architectures built around it is able to retrieve our base signals.}
% \label{lattice_anica}
% \end{figure}

\begin{figure}
\centering
%\subfigure[Iteration 0]{
%\includegraphics[width=0.3\linewidth]{anica_mieszanie_krata/ANICA_2D_pnl_0.png}
%}
%\\
\subfigure[PNL - Iteration 1]{
\includegraphics[width=0.22\linewidth]{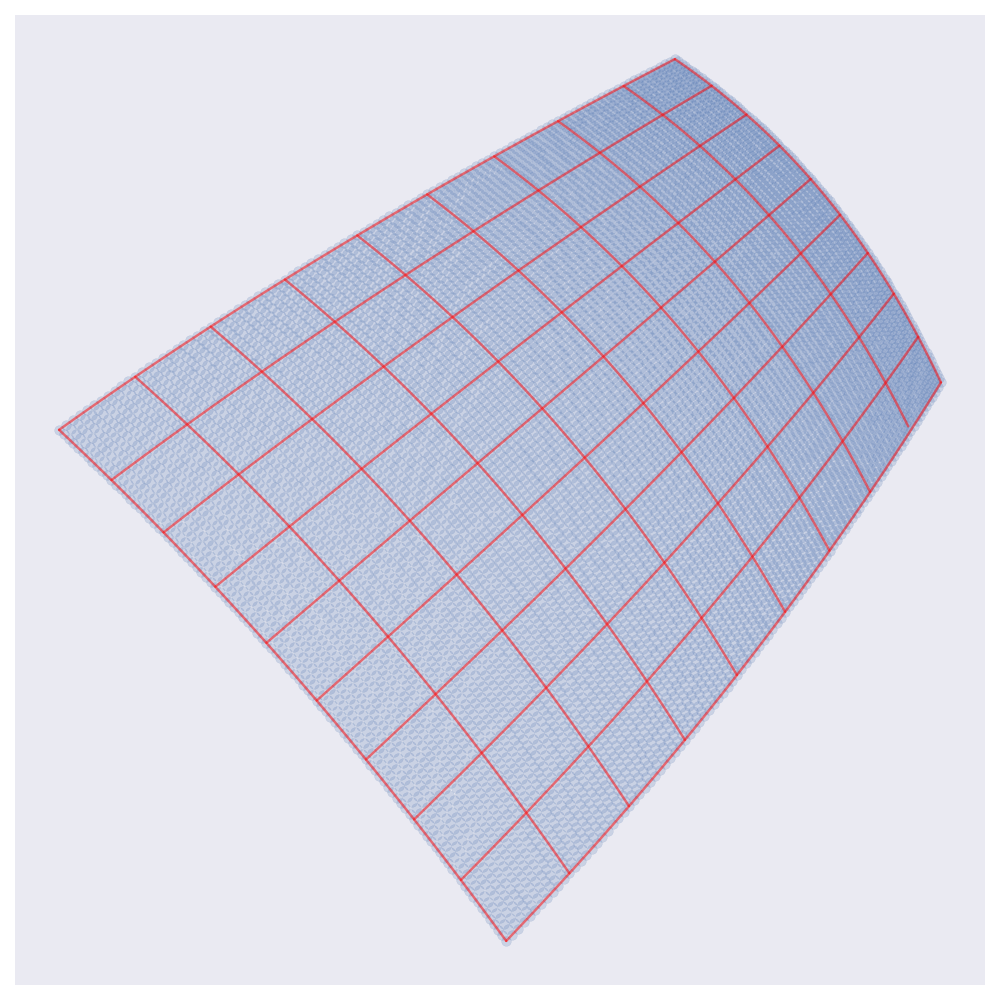}
}
\subfigure[PNL - Iteration 3]{
\includegraphics[width=0.22\linewidth]{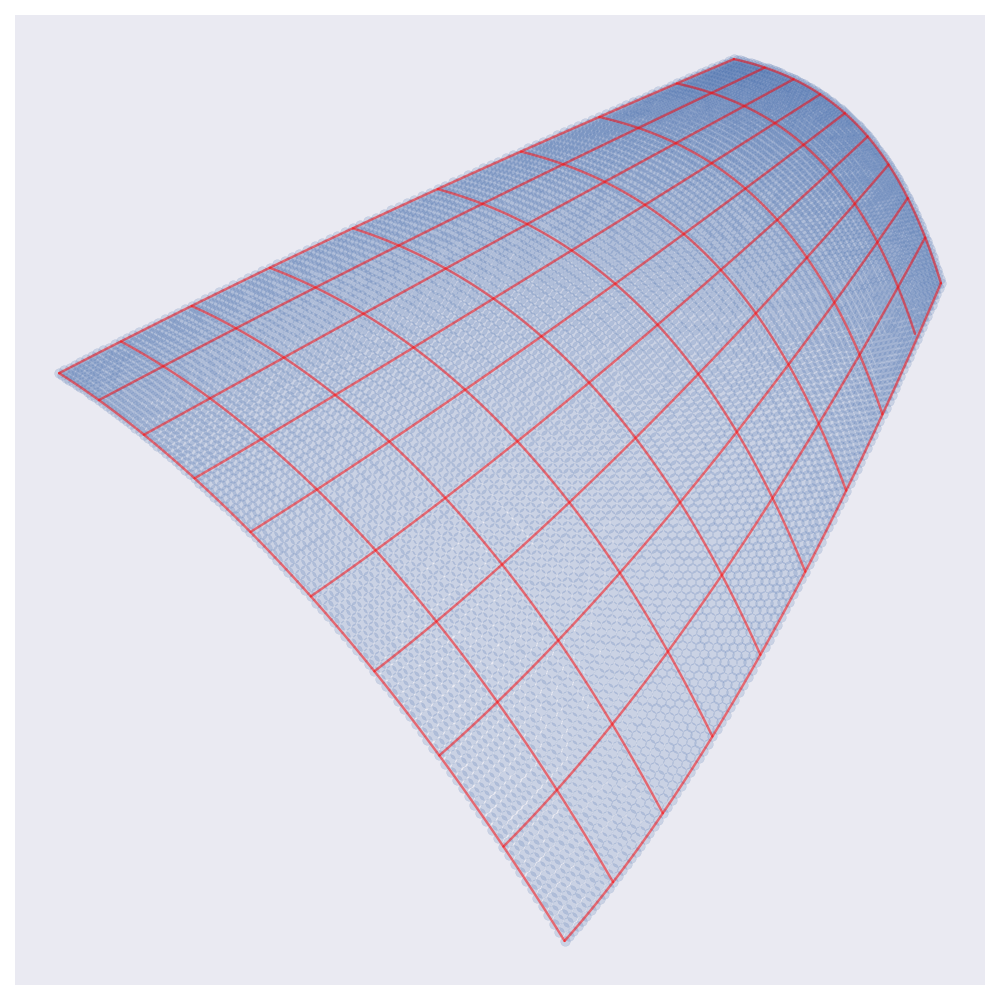}
}
\subfigure[MLP - Iteration 1]{
\includegraphics[width=0.22\linewidth]{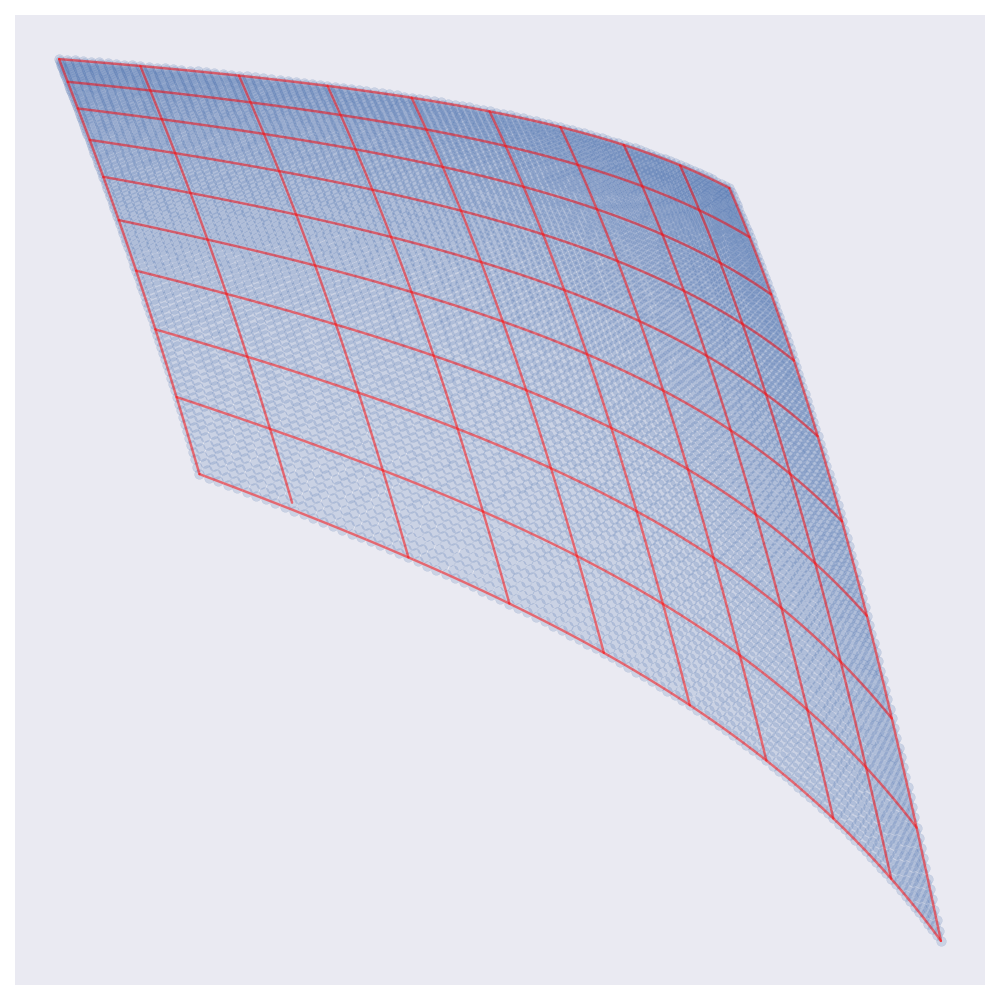}
}
\subfigure[MLP - Iteration 3]{
\includegraphics[width=0.22\linewidth]{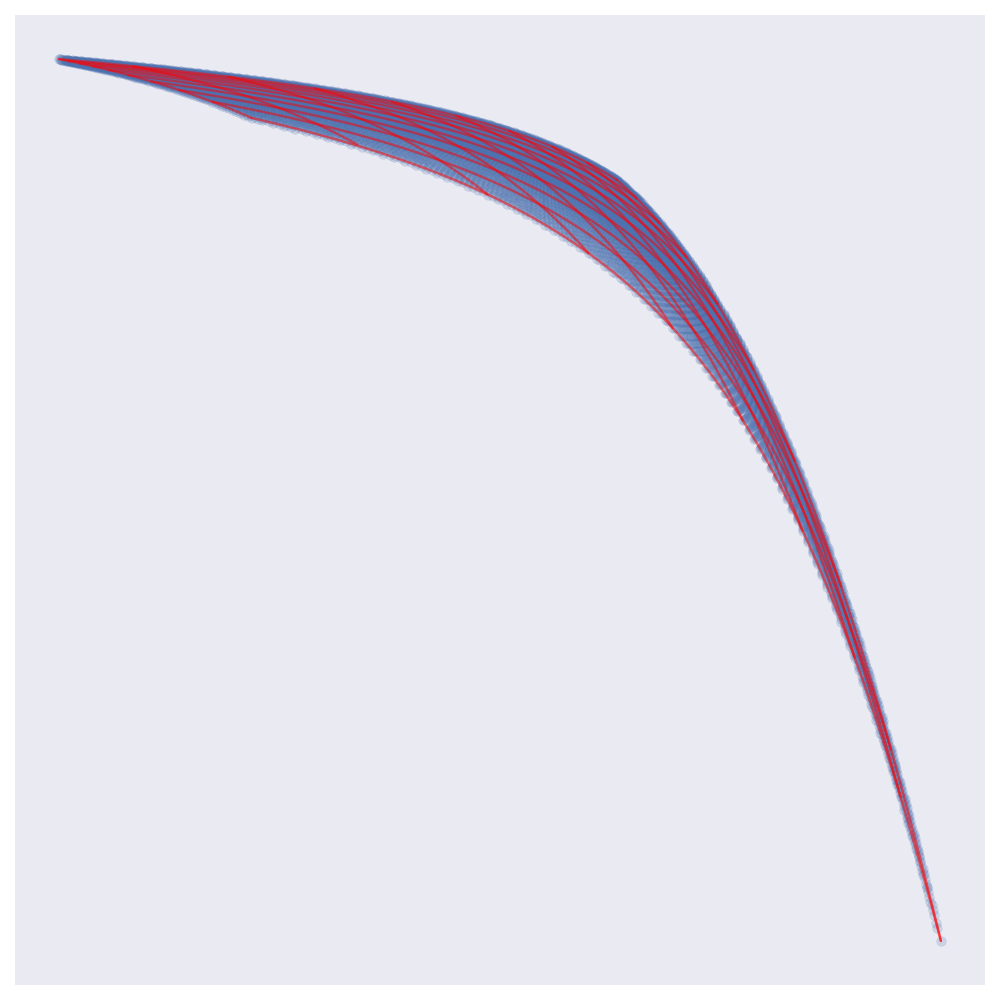}
}
\caption{Results of the nonlinear mixing techniques proposed in \Citep{brakel2017learning} on a normalized synthetic lattice data. Post nonlinear mixing model~(PNL) introduced only slight nonlinearities, which are not hard enough to solve even for the linear algorithms. On the other hand, the multi-layer perceptron mixing (MLP) technique collapses after just couple of iterations.} %In such a regime, only architectures built around it is able to retrieve our base signals.}
\label{lattice_anica}
\end{figure}

\begin{figure}
\centering
\subfigure[Iteration 0]{
\includegraphics[width=0.22\linewidth]{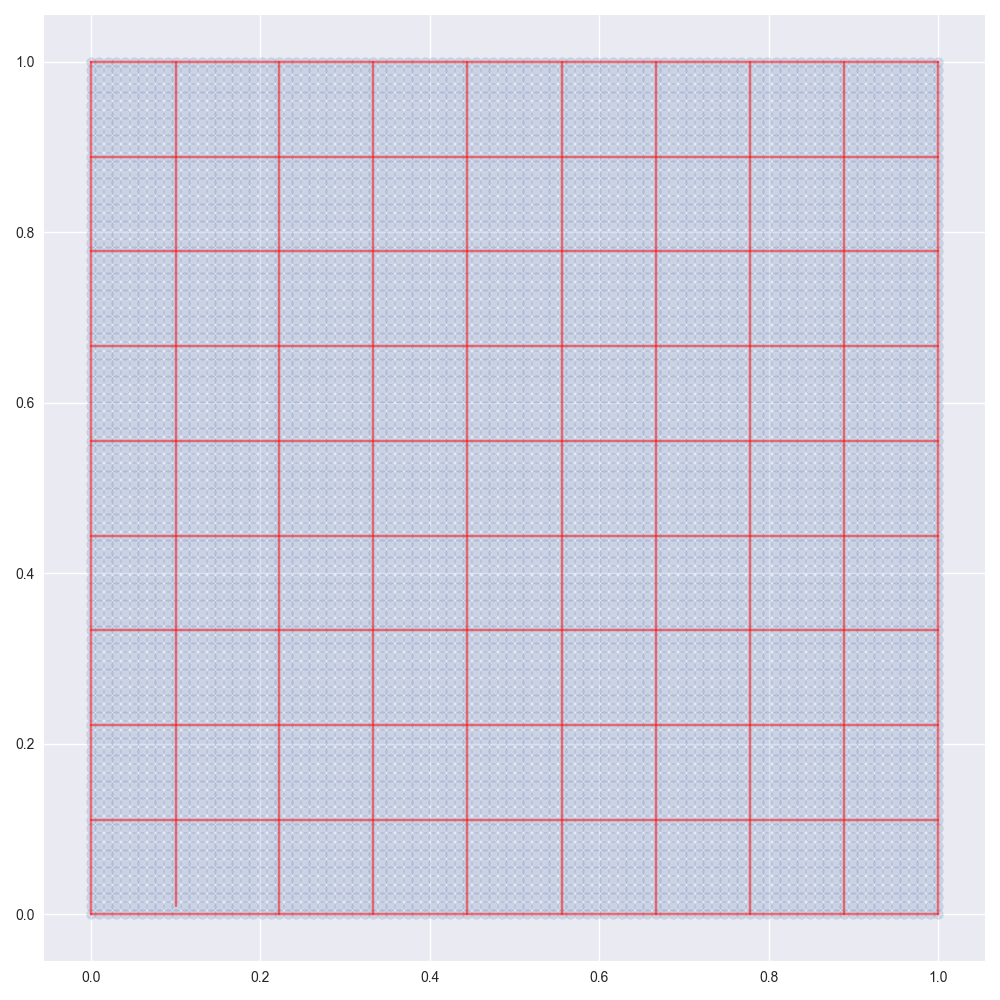}
}
\subfigure[Iteration 10]{
\includegraphics[width=0.22\linewidth]{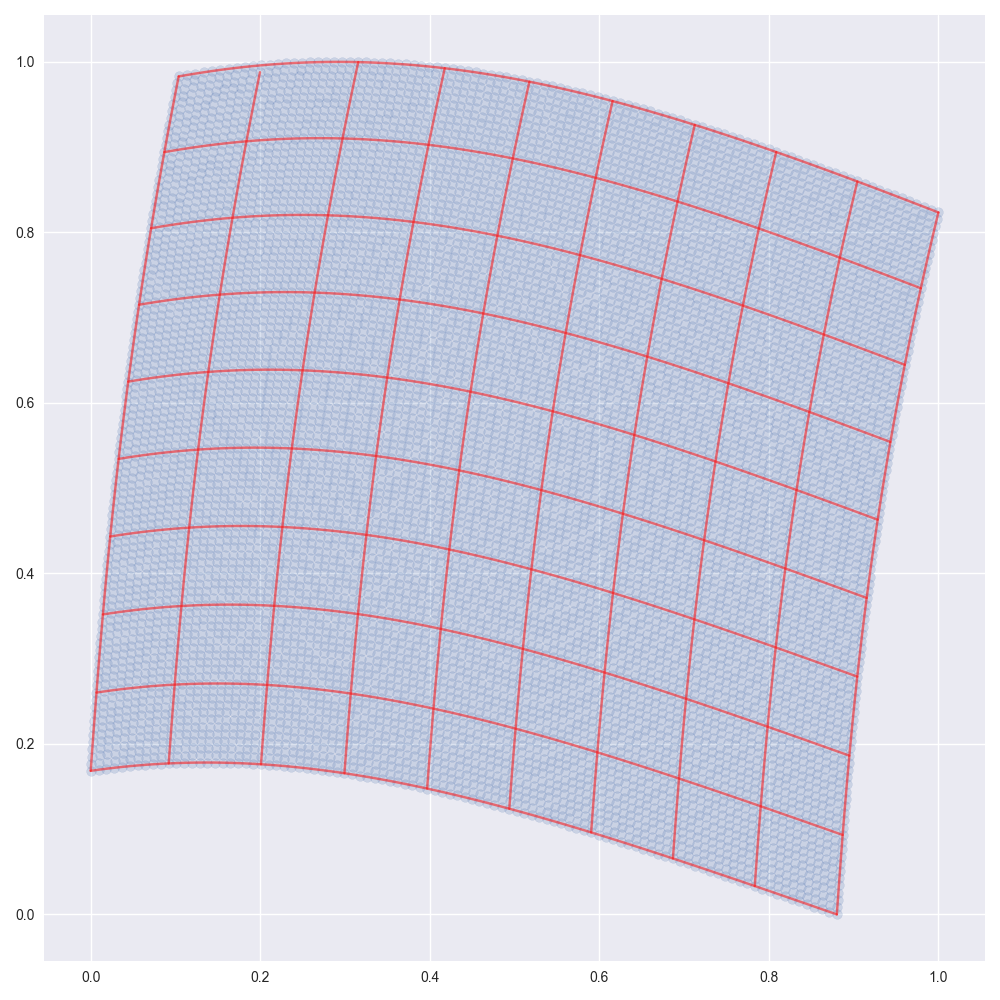}
}
\subfigure[Iteration 20]{
\includegraphics[width=0.22\linewidth]{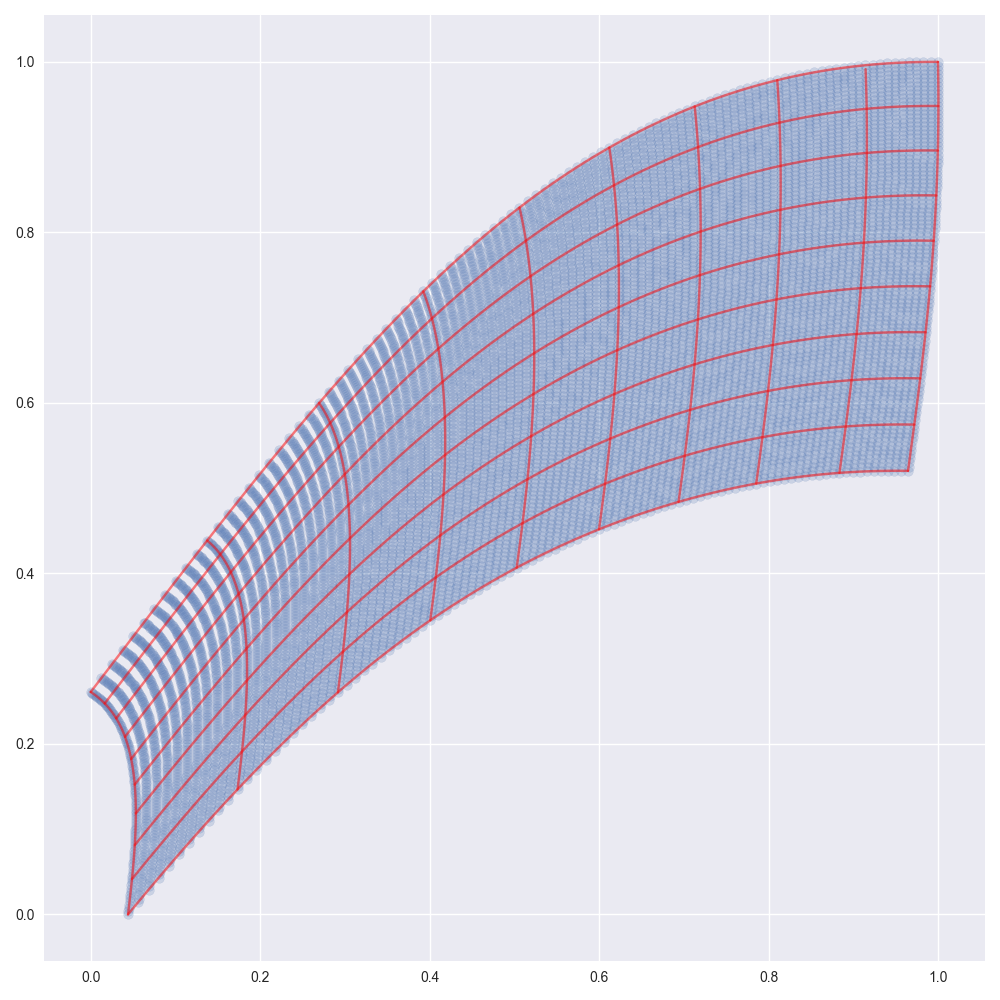}
}
\subfigure[Iteration 30]{
\includegraphics[width=0.22\linewidth]{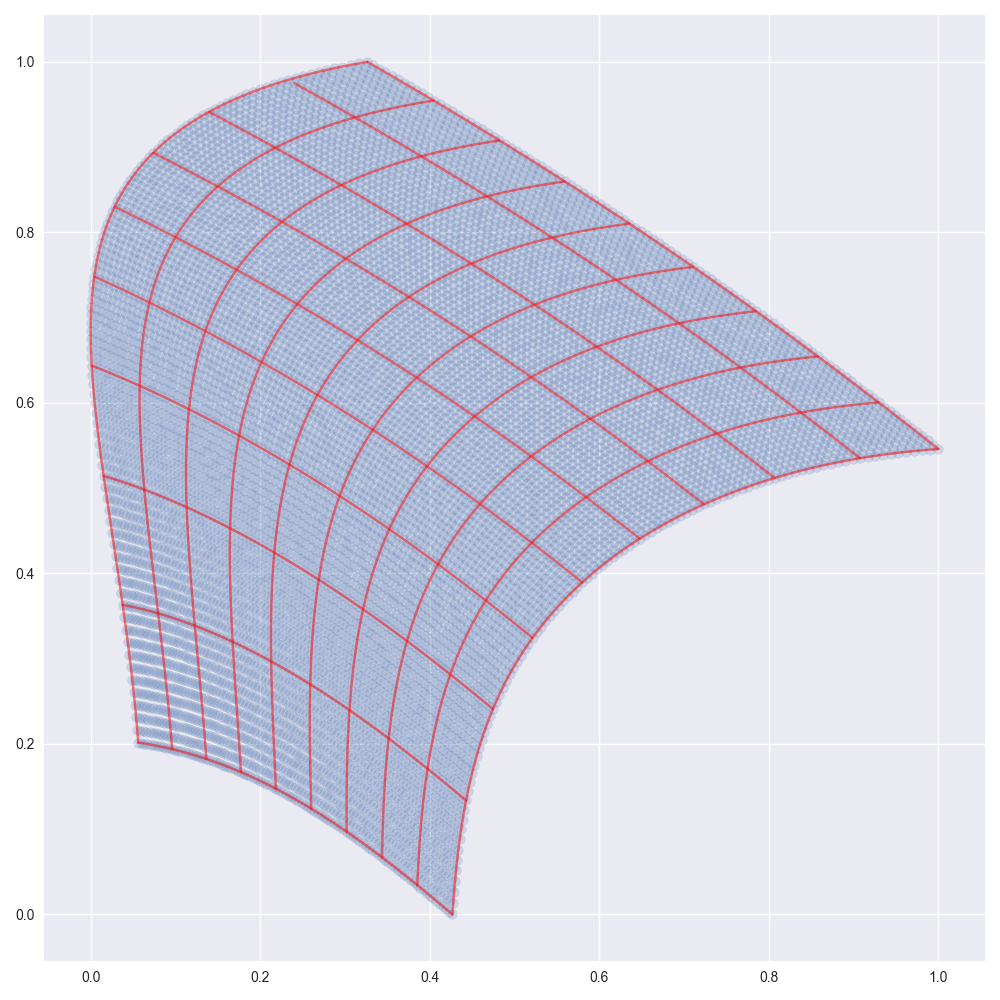}
}
\subfigure[Iteration 40]{
\includegraphics[width=0.22\linewidth]{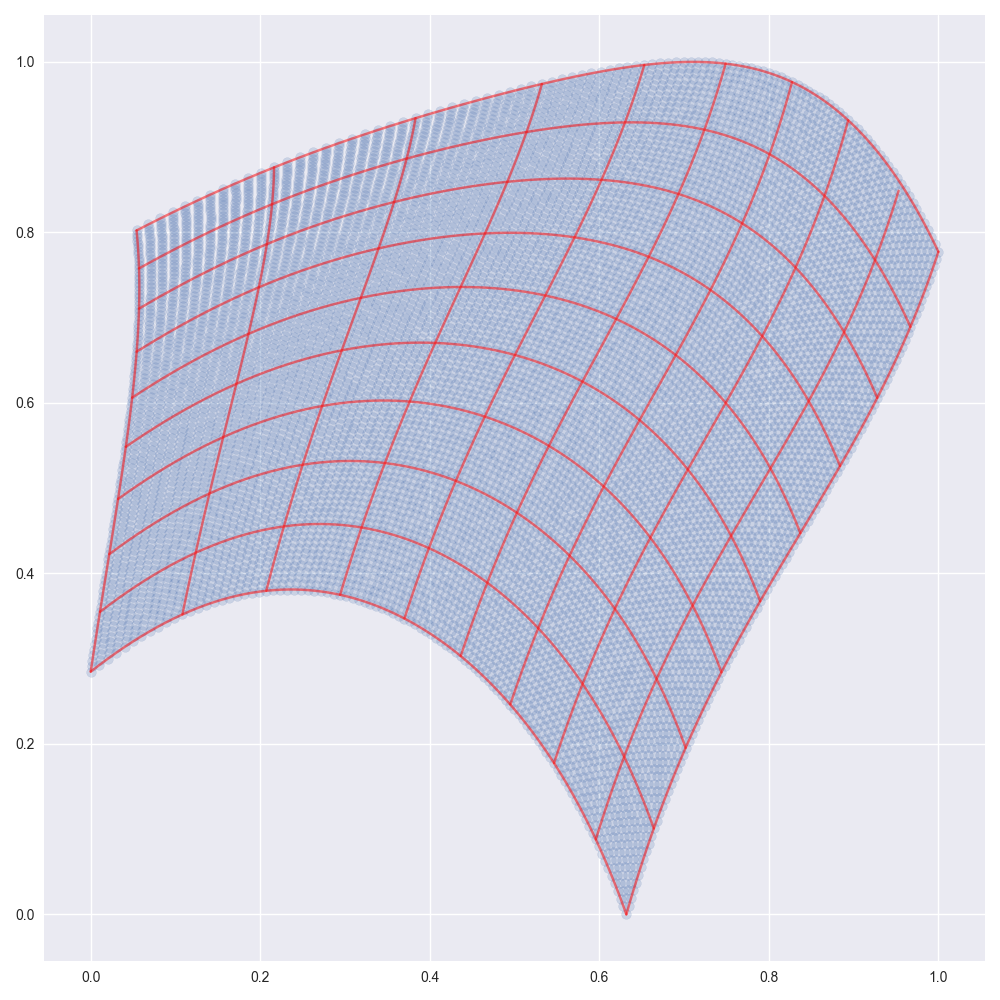}
}
\subfigure[Iteration 50]{
\includegraphics[width=0.22\linewidth]{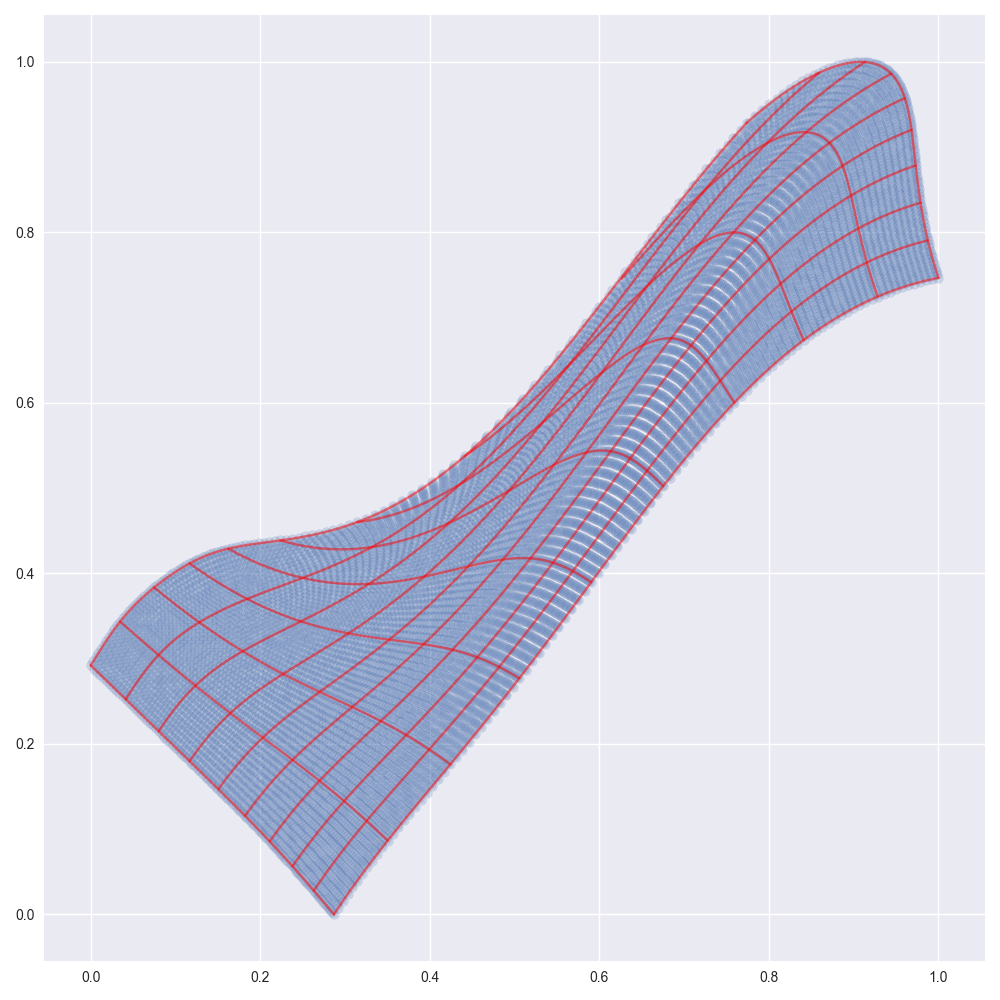}
}
\subfigure[Iteration 60]{
\includegraphics[width=0.22\linewidth]{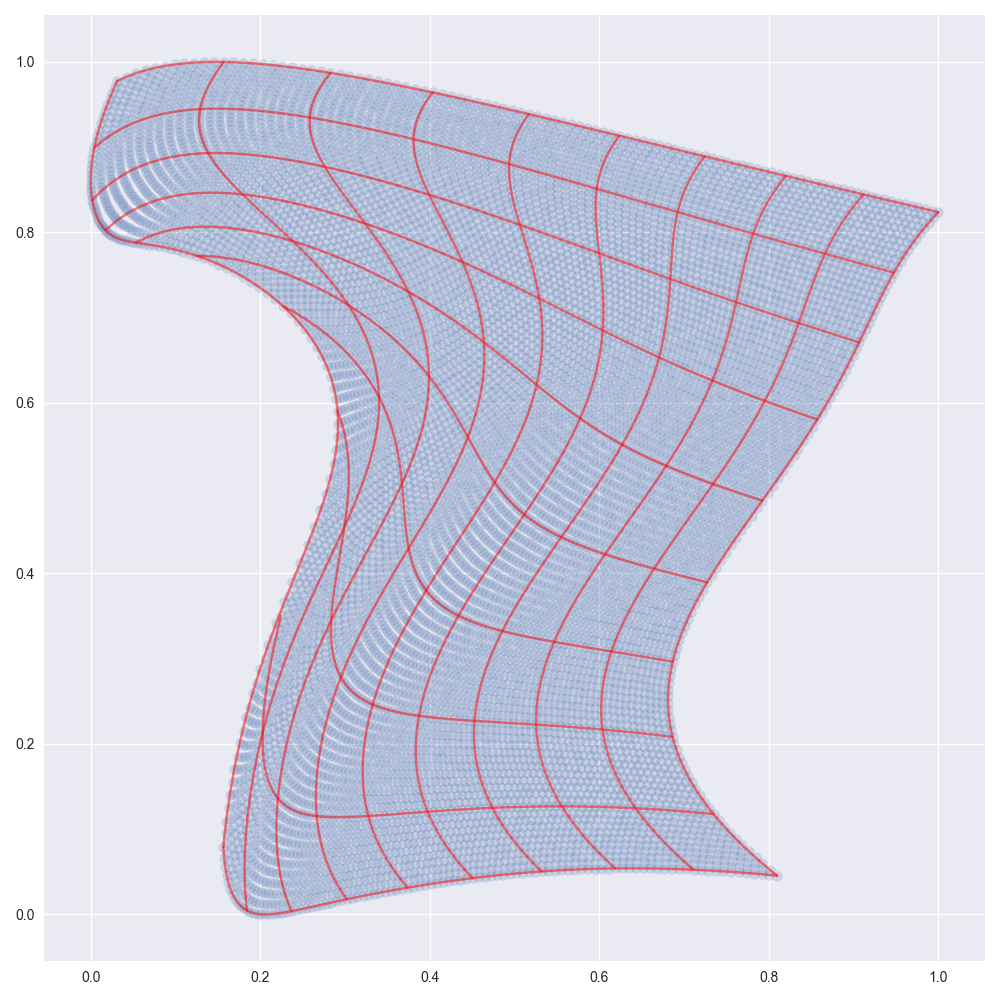}
}
\subfigure[Iteration 70]{
\includegraphics[width=0.22\linewidth]{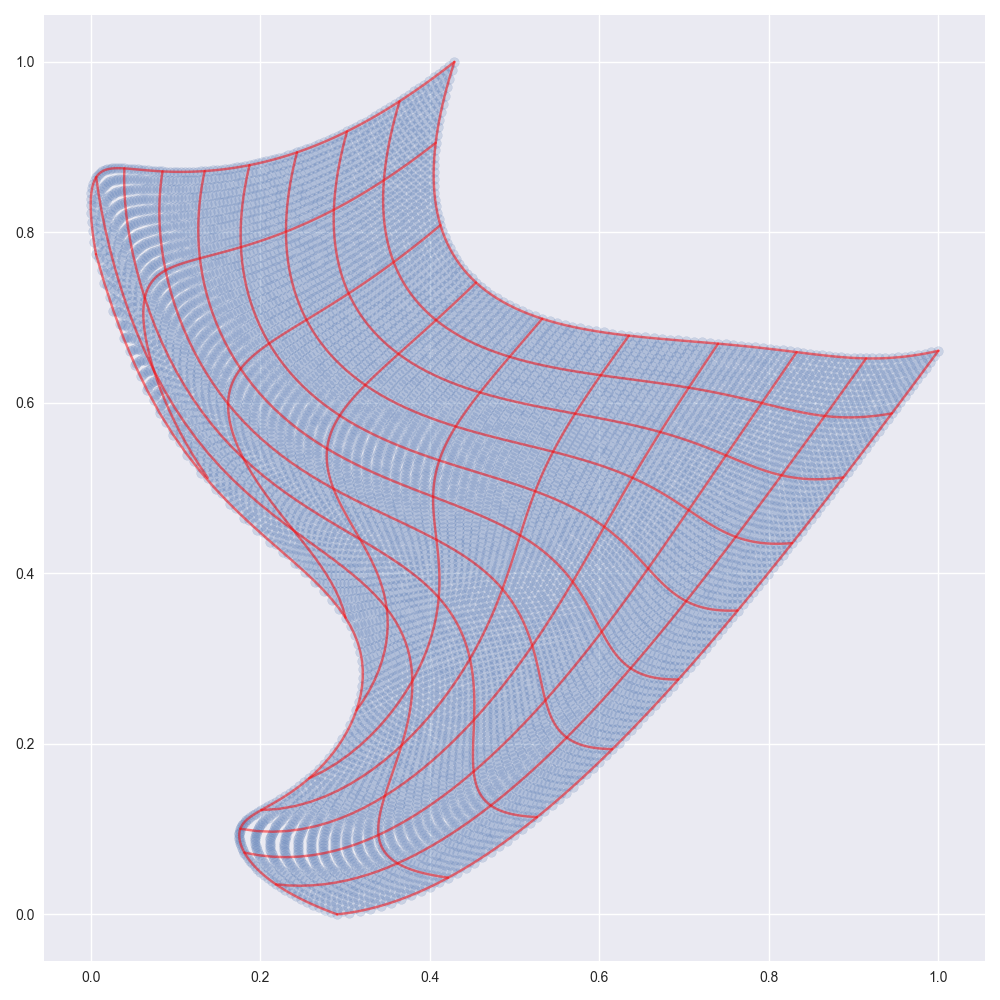}
}
% \subfigure[Iteration 80]{
% \includegraphics[width=0.3\linewidth]{odwzorowania_krata/dummy_data_neural_net_80.png}
% }
\caption{Results of our proposition of mixing over normalized synthetic lattice data. One may observe that after multiple iterations of the proposed mixing, results become highly nonlinear but not degenerate into any obscure solutions known from previous setup.
}\label{lattice}
\end{figure}

% \begin{figure}
% \centering
% \subfigure[Iteration 0]{
% \includegraphics[width=0.3\linewidth]{odwzorowania_krata/dummy_data_neural_net_0.png}
% }
% \subfigure[Iteration 10]{
% \includegraphics[width=0.3\linewidth]{odwzorowania_krata/dummy_data_neural_net_10.png}
% }
% \subfigure[Iteration 20]{
% \includegraphics[width=0.3\linewidth]{odwzorowania_krata/dummy_data_neural_net_20.png}
% }
% \subfigure[Iteration 30]{
% \includegraphics[width=0.3\linewidth]{odwzorowania_krata/dummy_data_neural_net_30.png}
% }
% \subfigure[Iteration 40]{
% \includegraphics[width=0.3\linewidth]{odwzorowania_krata/dummy_data_neural_net_40.png}
% }
% \subfigure[Iteration 50]{
% \includegraphics[width=0.3\linewidth]{odwzorowania_krata/dummy_data_neural_net_50.png}
% }
% \subfigure[Iteration 60]{
% \includegraphics[width=0.3\linewidth]{odwzorowania_krata/dummy_data_neural_net_60.png}
% }
% \subfigure[Iteration 70]{
% \includegraphics[width=0.3\linewidth]{odwzorowania_krata/dummy_data_neural_net_70.png}
% }
% \subfigure[Iteration 80]{
% \includegraphics[width=0.3\linewidth]{odwzorowania_krata/dummy_data_neural_net_80.png}
% }
% \caption{
% Results of our proposition of mixing over synthetic lattice data. One may observe that after multiple iterations of the proposed mixing, results become highly nonlinear but not degenerate into any obscure solutions known from previous setup.
% }\label{lattice}
% \end{figure}

Because of aforementioned disadvantages we propose our own mixing, inspired by \Citep{glow, dinh2014nice} network architecture. 
Let $S$ be a sample of~vectors with independent components. We apply a random isometry on $S$, by taking $X=\left(UV^T\right)S$, where $UV^T$ comes from the Singular~Value~Decomposition on a random matrix $A_{ij}\sim \nor(0,1)$. Next we split $X \in \R^d$ into half
\begin{equation*}
        (x_i,x_j) \to \left(x_i,x_j+\phi(x_i)\right) \label{eq:nonlin_mixing},
\end{equation*}

similarly as it was done in \Citep{glow}. Function $\phi$ is a randomly initialized neural network with two hidden layers and $\tanh$ activations after each of them. This approach can be iterated over multiple times to achieve the desired level of nonlinear mixing.

% \ola{We start this section with a discussion of possible definitions of the nonlinear mixing function used for benchmarking the ICA methods. In the beginning we shortly explain some approaches used in the linear ICA, and then move forward to propose a mixing which benefits from properties desired in the comparison of the results obtained by nonlinear ICA algorithms.}  
%We start this section from the discussion of nonlinear approach to mixing signals for benchmarking ICA methods. In the beginning we shortly explain of the linear approach, and than move forward to propose mixing which has properties desired to compare results of nonlinear ICA algorithms.

Mixing procedure can be described in an algorithmic way:

\begin{algorithm}[H]  
\caption{Nonlinear mixing}
\begin{flushleft}
Take dataset $S_0$.
\end{flushleft}
\begin{enumerate}
    \item Take random isometry:
    \begin{enumerate}
        \item Take $A$ from $\nor{(0,\I)}$, such that
        $a_{ij} \sim \nor{(0,1)}$
        \item Take SVD of $A$, such that $A=U\Sigma V^T$
        \item Return $UV^T$ 
    \end{enumerate}
    \item Take $X=\left(UV^T\right)S_0$
    \item Split $X \in \R^D$ in half:
    $$
    (x_i,x_j) \to \left(x_i,x_j+\phi(x_i)\right)
    $$
    where $\phi$ is a randomly initialized neural network and $x_i, x_j$ come from the split of~$X$ into half.
    \item Return $X_1=AX$
\end{enumerate}
\end{algorithm}\label{alg:nonlin_mixing}
One can easily increase the number of mixes and interlude splits of~$X$ in~reverse
order so that $(x_i,x_j) \to \left(x_i+\phi(x_j), x_j\right)$ for even and \break $(x_i,x_j) \to \left(x_i,x_j+\phi(x_i)\right)$ for odd iterate. The effects of applying the proposed mixing to two-dimensional data are presented in Fig.~\ref{lattice}.

Our mixing procedure scales well in higher dimensions by iterating over the splits in $\R^d$. Additionally, it is also easily invertible, therefore there is a guarantee that the source components may be~retrieved. 

\begin{figure*}[!h]

\centering
  \begin{tabular}[t]{cc}
    2 dimensions & 4 dimensions \\ 
  \includegraphics[width=0.49\linewidth]{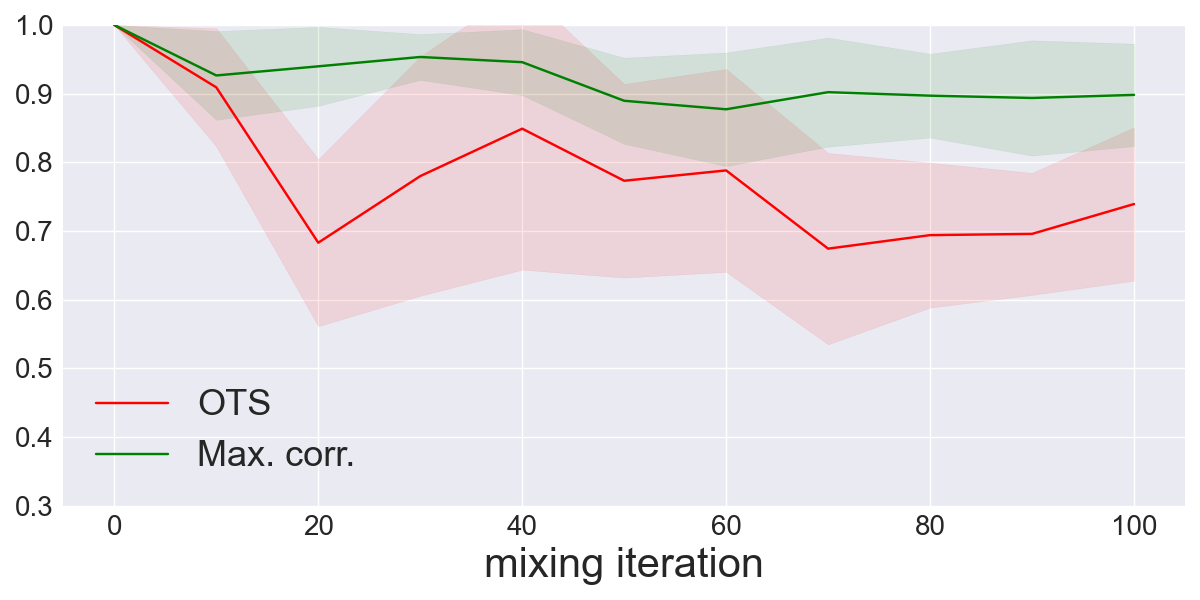}
  &
\includegraphics[width=0.49\linewidth]{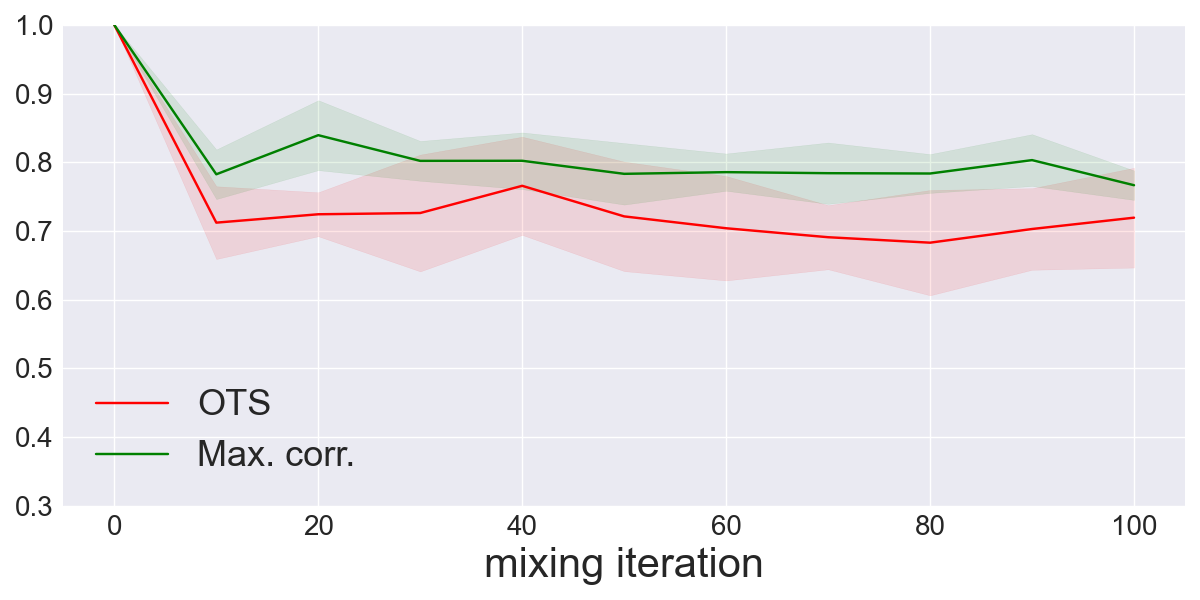}
 \\
 6 dimensions & 8 dimensions \\
\includegraphics[width=0.49\linewidth]{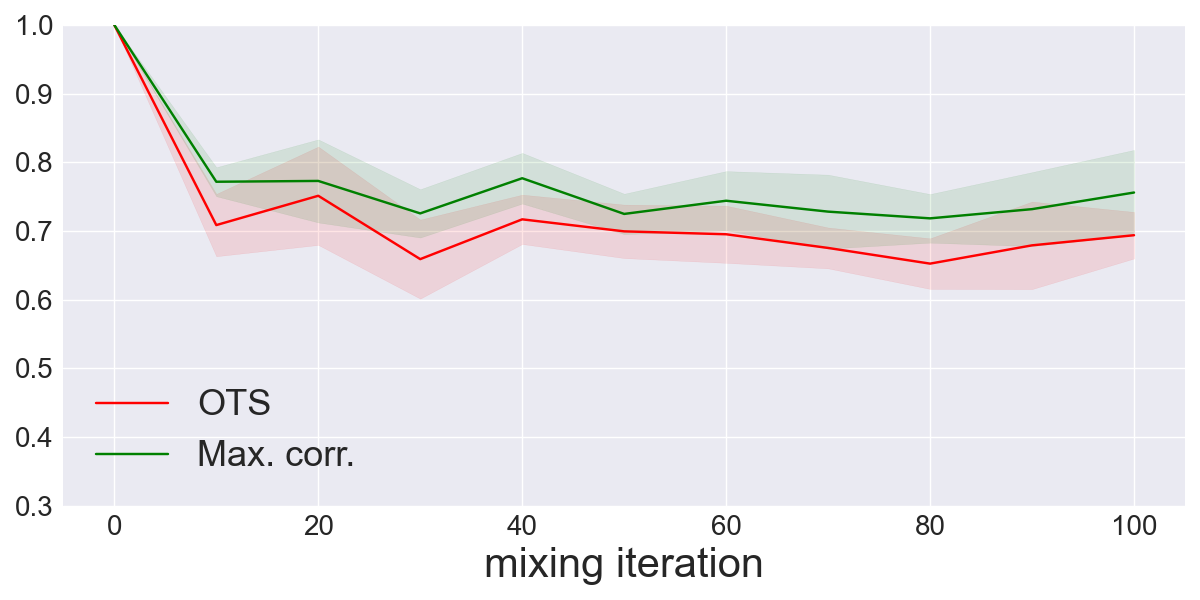} 
&
\includegraphics[width=0.49\linewidth]{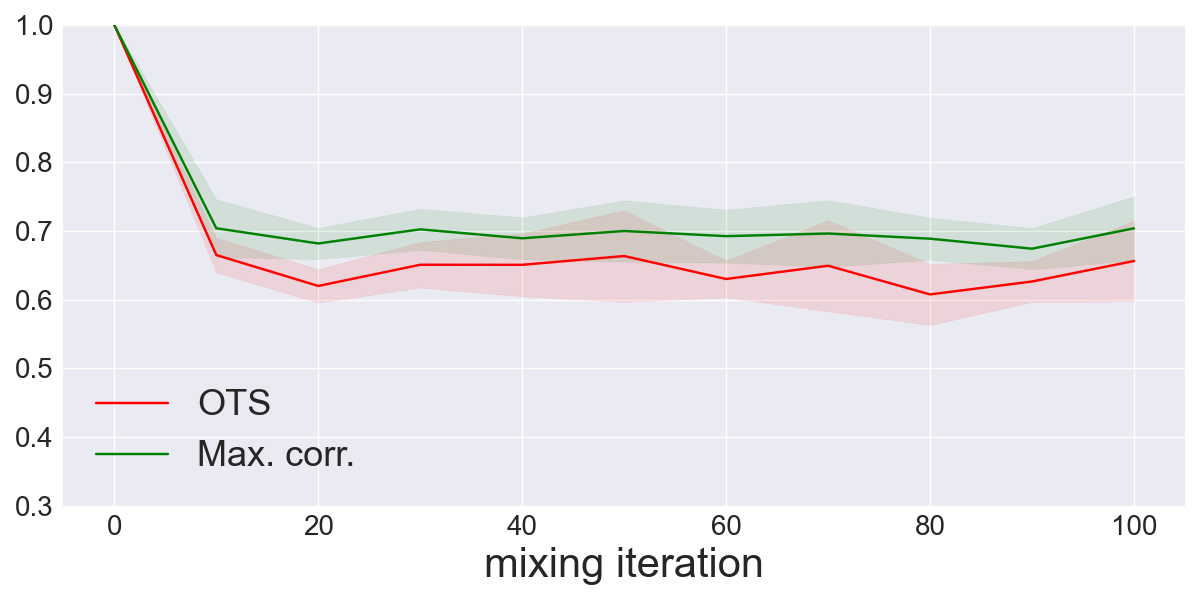}
\end{tabular}
\caption{Results of the experiment where in $n$--dimensional mixed observation one component was swapped with a randomly chosen source signal. One may observe that $max\_corr$ almost always prefers such situation, while \mixname{} seems to be more rigorous. %It can be interpreted in the way, that $max\_corr$ prefers ideally denoised variable at the expense of other components from the base signals dataset. On the other hand \mixname{} prefer the signal to be more coherent with all of the original components.
}\label{spearmanVSmaxcorr}
\end{figure*}

\begin{figure*}[!h]
\centering
  \begin{tabular}[t]{cc}
    2 dimensions & 4 dimensions \\ 
  \includegraphics[width=0.49\linewidth]{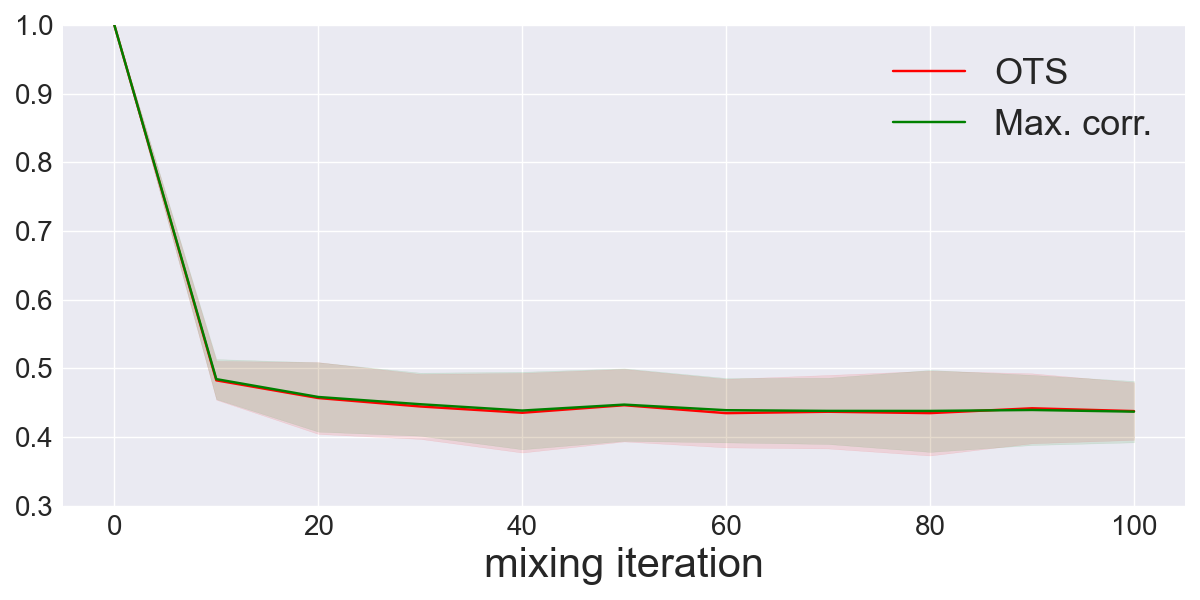}
  &
\includegraphics[width=0.49\linewidth]{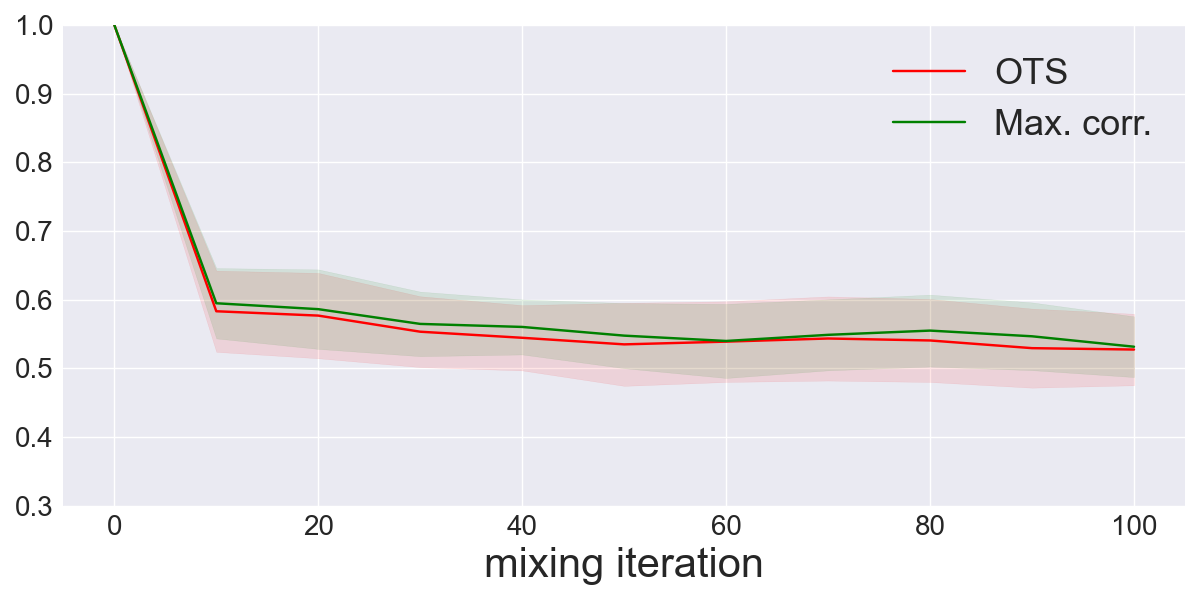}

 \\
 6 dimensions & 8 dimensions \\
\includegraphics[width=0.49\linewidth]{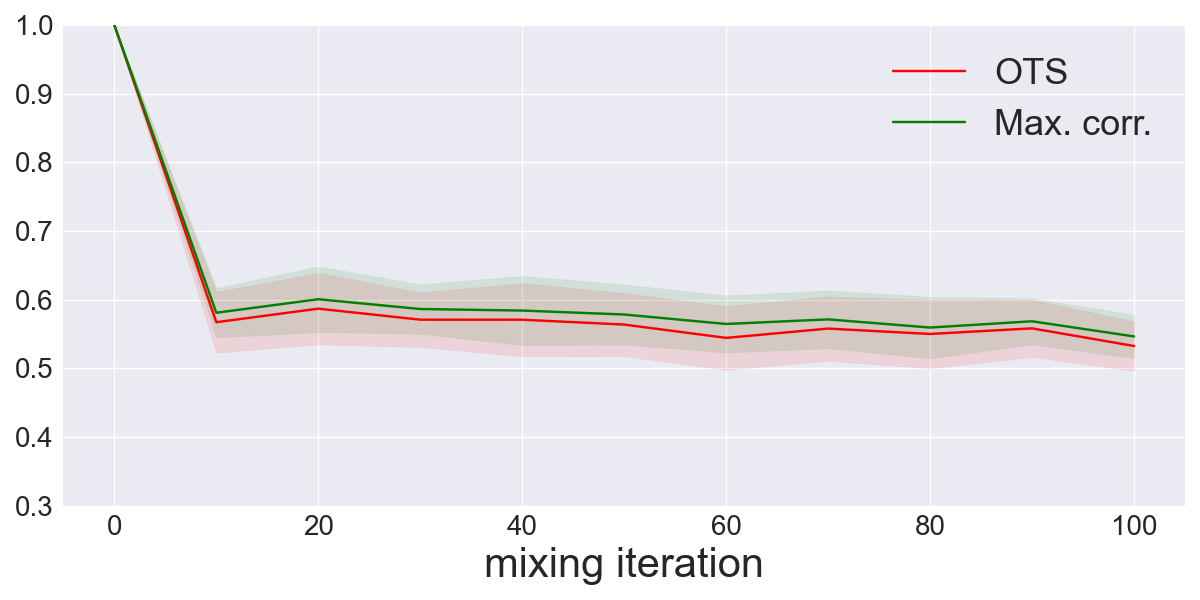}
 &
  \includegraphics[width=0.49\linewidth]{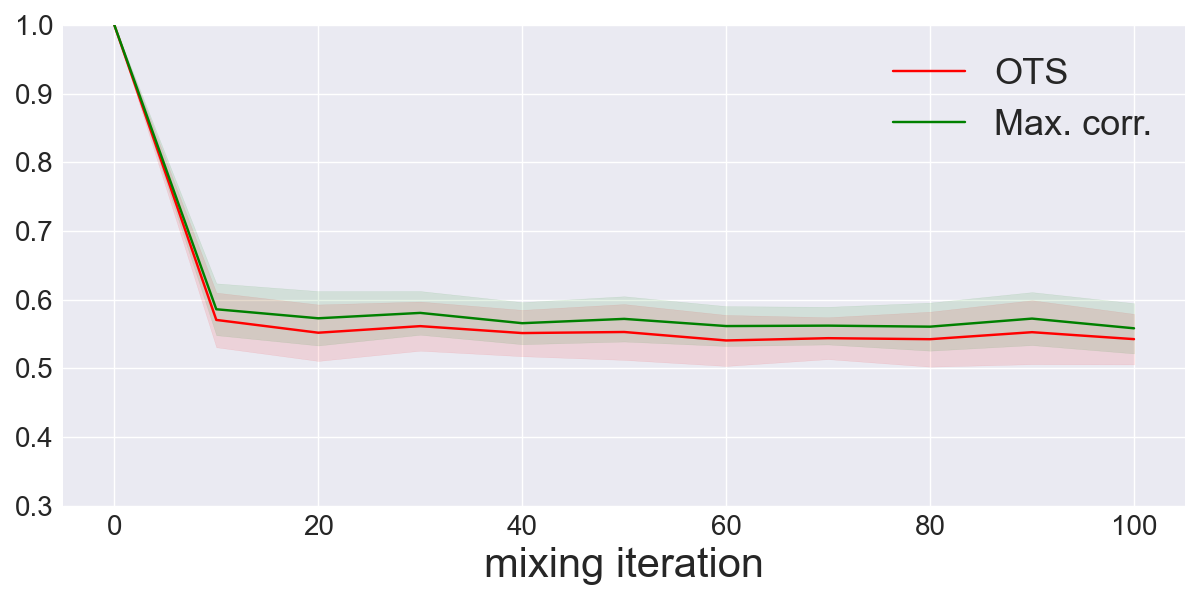}
  \\
% \subfigure[Measures comparison on lattice]{\includegraphics[width=0.2\linewidth]{eksperyment_z_miarami/krata_searman_vs_max_corr_dim_2.png}
% \includegraphics[width=0.2\linewidth]{eksperyment_z_miarami/krata_searman_vs_max_corr_dim_4.png}
% \includegraphics[width=0.2\linewidth]{eksperyment_z_miarami/krata_searman_vs_max_corr_dim_6.png}
% \includegraphics[width=0.2\linewidth]{eksperyment_z_miarami/krata_searman_vs_max_corr_dim_8.png}
% \includegraphics[width=0.2\linewidth]{eksperyment_z_miarami/krata_searman_vs_max_corr_dim_10.png}
% }
% \subfigure[Measures comparison on mixed images]{
% \includegraphics[width=0.2\linewidth]{eksperyment_z_miarami/obrazki_searman_vs_max_corr_dim_2.png}
% \includegraphics[width=0.2\linewidth]{eksperyment_z_miarami/obrazki_searman_vs_max_corr_dim_4.png}
% \includegraphics[width=0.2\linewidth]{eksperyment_z_miarami/obrazki_searman_vs_max_corr_dim_6.png}
% \includegraphics[width=0.2\linewidth]{eksperyment_z_miarami/obrazki_searman_vs_max_corr_dim_8.png}
% \includegraphics[width=0.2\linewidth]{eksperyment_z_miarami/obrazki_searman_vs_max_corr_dim_10.png}
% }
\end{tabular}

% \andrzej{Przeliczyłem teraz wszystkie te eksperymenty (każde mieszanie ma po 30 seedów - stąd dużo mniejsze odchylenia). Zwiększyłem rozmiar thicków i labelek - jakby trzeba było bardziej to daj znać. Przeplotuje je rano.}
\caption{Results for the \mixname{} and $max\_corr$ values for fully mixed dataset. One may observe that both measures in this case give similar outcomes.
%All values are also significantly lower than those from experiment highlighted in Fig. \ref{spearmanVSmaxcorr}, this may be interperted as a proof of the independence measurement action. 
}\label{fullMixing}
\end{figure*}

\section{Optimal Transport Spearman measure} 
\label{sec:OTS}

For the~benchmark experiments we want to be able to measure the~similarity between the obtained results $Z$ and the original sources~$S$. In the case of linear mixing the common choice is the maximum absolute correlation over all possible permutations of the signals (denoted hereafter as $max\_corr$ \Citep{hyvarinen2016unsupervised,hyvarinen2017nonlinear, hyvarinen2019nonlinear,spurek2020non, pnlmisep,bengio2013representation, hyvarinen1999fast}).

However, this measure is based on the Pearson's correlation coefficient and therefore is not able to catch any high order dependencies. To address this problem we introduce a new measure based on the~nonlinear Spearman's rank correlation coefficient and optimal transport. 

Let the $Z$ denote the signal retrieved by an ICA algorithm and let the  $r_s\left(z^j, s^k\right)$ be the Spearman's rank correlation coefficient between the $j$-th component of $Z$ and $k$-th component of $S$. We define the~Spearman's distance matrix $M(Z,S)$ as
$$
M(Z,S)=\left[1-\left|r_s(z^j,s^k)\right|\right]_{j,k=1,2,\ldots d},
$$
where the zero entries indicate a monotonic relationship between the~corresponding features. 

This matrix is then used as the transportation cost of the components. Formally, we compute the value of the optimal transport problem formulated in terms of integer linear programming:
$$     
     \mixname = 1 - I_s(Z,S),
     $$
     $$
     I_s(Z,S) = \min_{\gamma} \frac{1}{D}\sum_{j,k} \gamma_{j,k}M(Z,S)_{j,k},
$$
subject to:
$$\sum_{k}^{d}\gamma_{j,k} = A_j\mathrm{\ for\ all\ }j\in\{1,2,\ldots,d\}, $$
$$ \sum_{j}^{d}\gamma_{j,k} = A_k\mathrm{\ for\ all\ }k\in\{1,2,\ldots,d\},$$
$$ \gamma_{j,k}\in \{0,1\}\mathrm{\ for\ all\ }j,k\in\{1,2,\ldots,d\},  $$
where $A_j=A_k=1$. 

As a result of the last constraint, the obtained transport plan~$\gamma$~defines a one-to-one map from the retrieved signals~to~the~original sources. In addition, the proposed Spearman-based measure ($\mixname{}$) is sensitive to monotonic nonlinear dependencies and also relatively easy to compute with the use of existing tools for integer programming.

Another difference between \mixname{} and  $max\_corr$ is that the latter favors stronger disentanglement of few components, while \mixname{} gives lower results for outcomes that decompose the observation more equally. In other words consider an experiment in which $n$ signals were mixed. Further, assume that some (nonlinear) ICA algorithm failed to unmix all but one component (i.e. only one unmixed component matches exactly one source signal, while the rest is still highly unrecognizable). In such situation the $max\_corr$ value will be significantly higher than \mixname{}, although only the small portion of the base dataset was recovered.

\begin{figure}[!h]
\centering
\rotatebox{90}{\hskip 0.1in Original }
\includegraphics[width=0.25\linewidth]{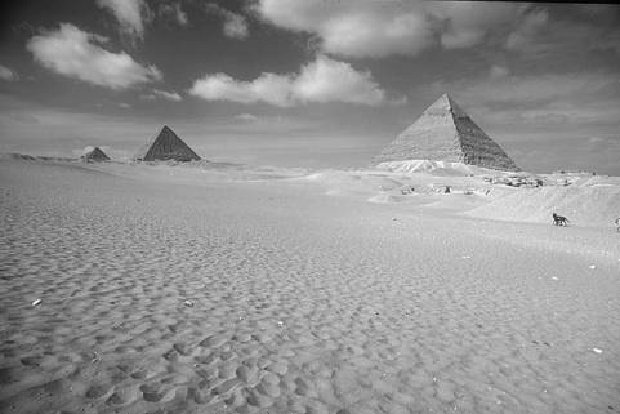}
\includegraphics[width=0.25\linewidth]{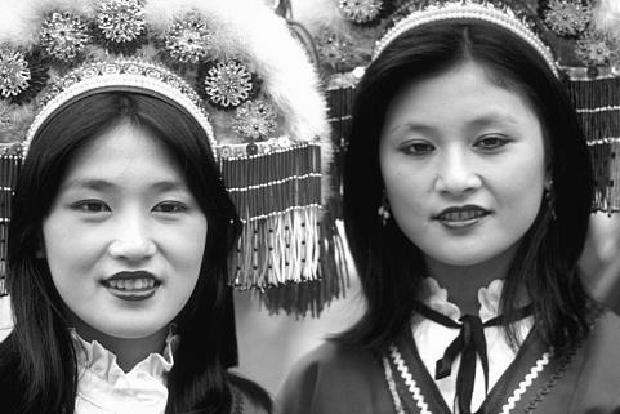}
\includegraphics[width=0.25\linewidth]{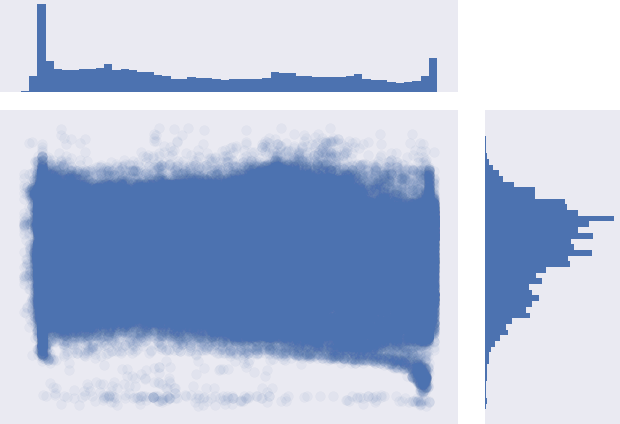}\\
\rotatebox{90}{\hskip 0.1in Mixed}
\includegraphics[width=0.25\linewidth]{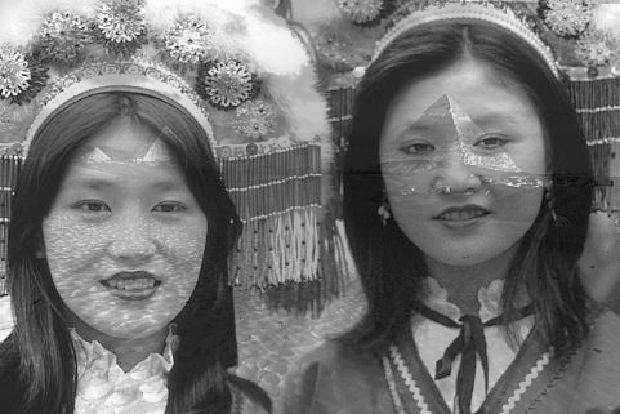}
\includegraphics[width=0.25\linewidth]{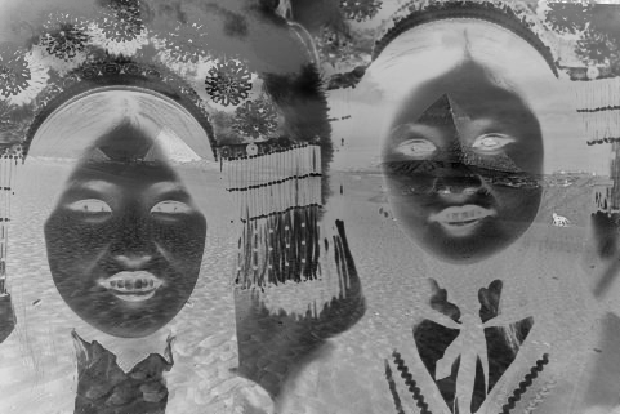}
\includegraphics[width=0.25\linewidth]{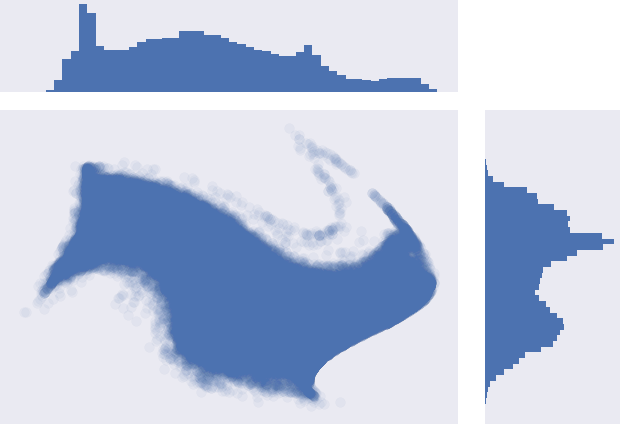}\\
\rotatebox{90}{\hskip 0.1in FastICA}
\includegraphics[width=0.25\linewidth]{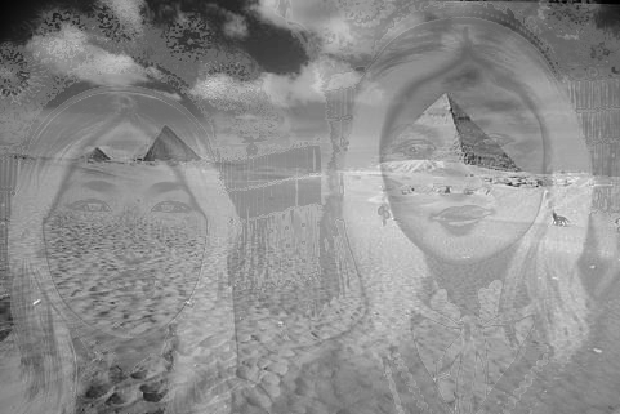}
\includegraphics[width=0.25\linewidth]{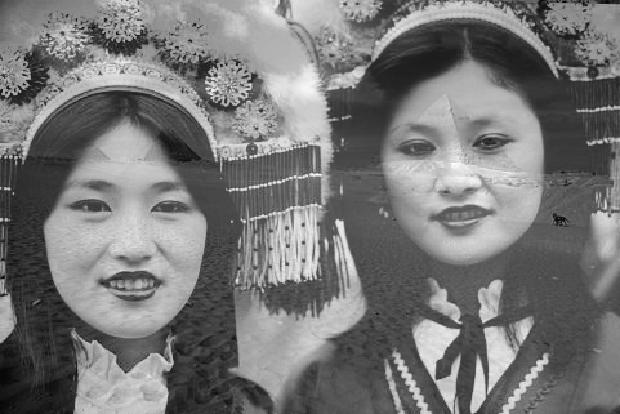}
\includegraphics[width=0.25\linewidth]{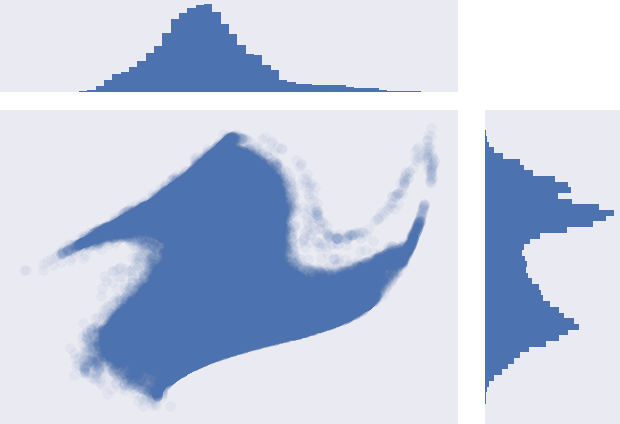}\\
\rotatebox{90}{\hskip 0.1in ANICA}
\includegraphics[width=0.25\linewidth]{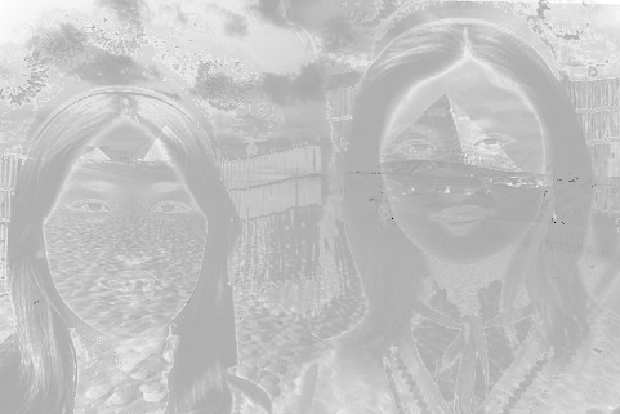}
\includegraphics[width=0.25\linewidth]{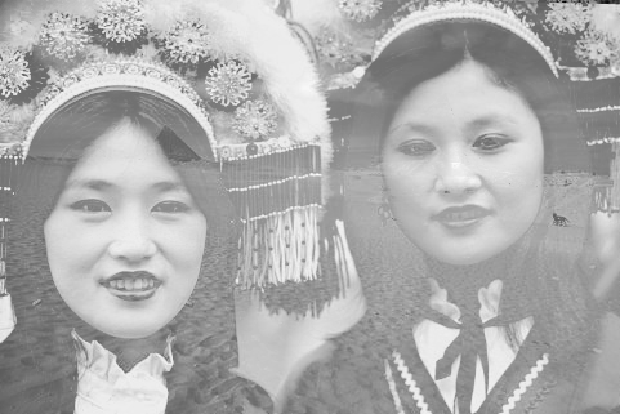}
\includegraphics[width=0.25\linewidth]{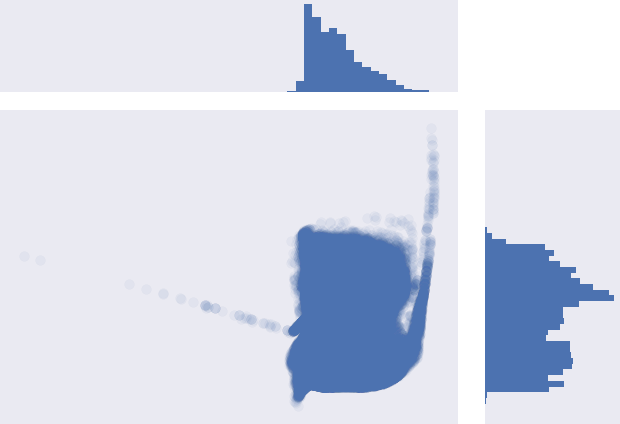}\\
\rotatebox{90}{\hskip 0.2in  PNLMISEP}
\includegraphics[width=0.25\linewidth]{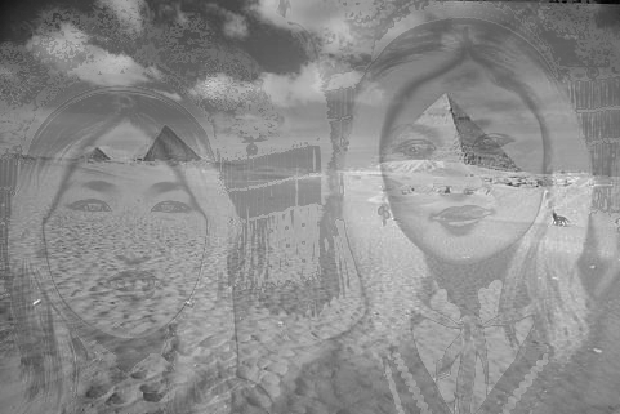}
\includegraphics[width=0.25\linewidth]{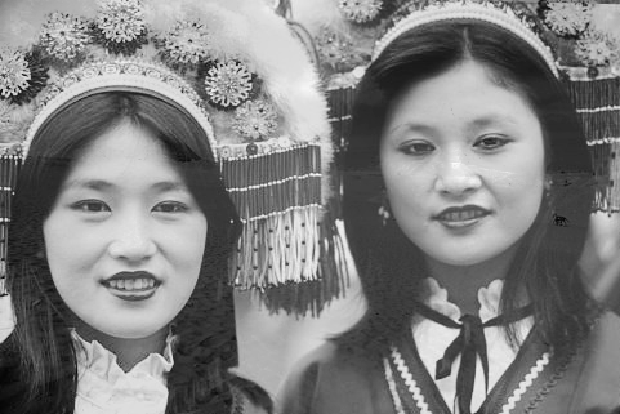}
\includegraphics[width=0.25\linewidth]{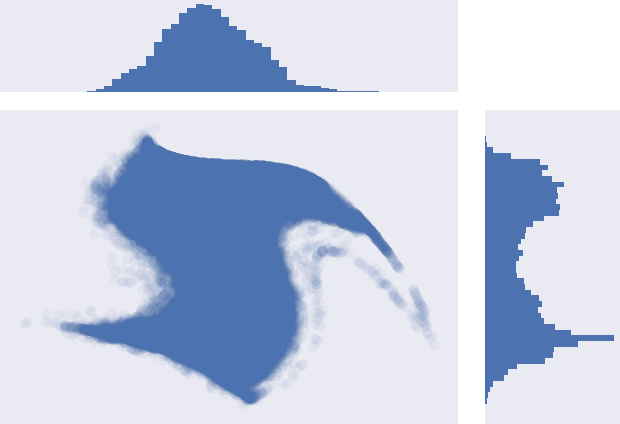}\\
\rotatebox{90}{\hskip 0.1in dCor}
\includegraphics[width=0.25\linewidth]{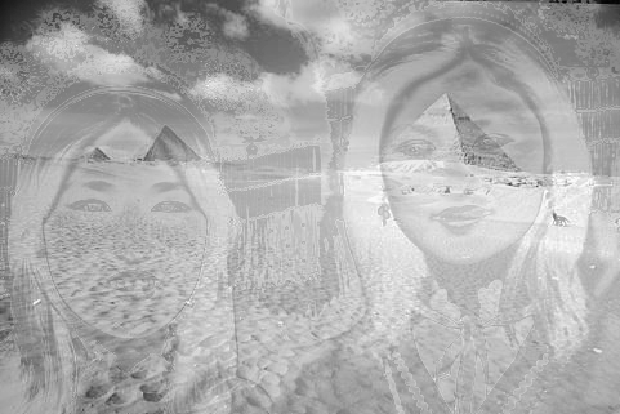}
\includegraphics[width=0.25\linewidth]{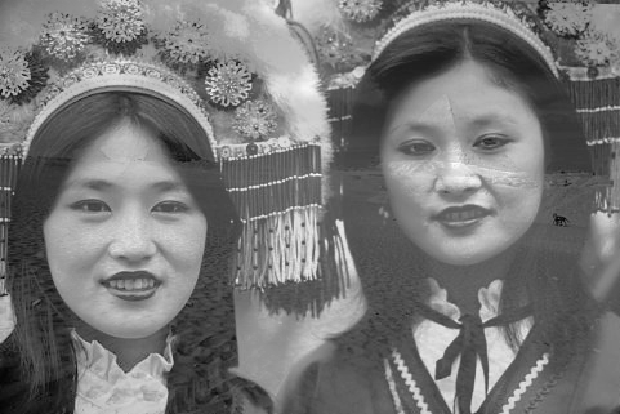}
\includegraphics[width=0.25\linewidth]{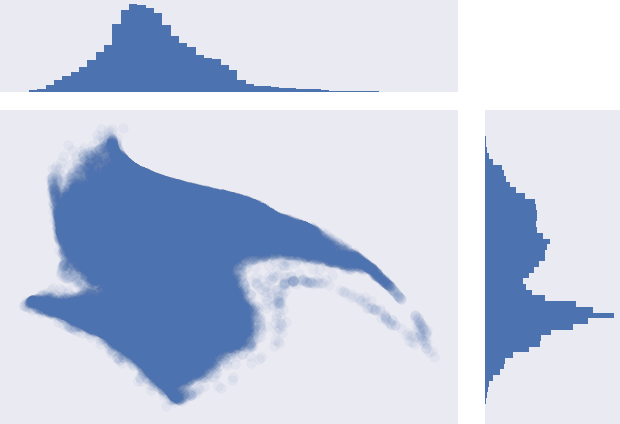}\\
\rotatebox{90}{\hskip 0.1in \wica{}}
\includegraphics[width=0.25\linewidth]{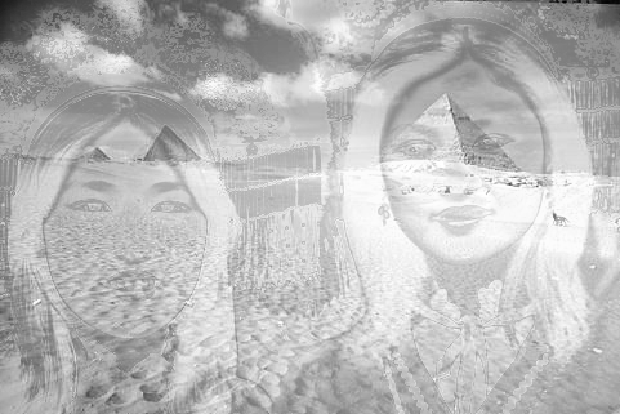}
\includegraphics[width=0.25\linewidth]{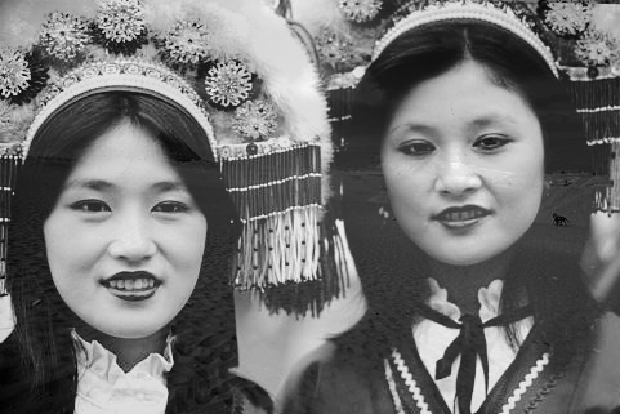}
\includegraphics[width=0.25\linewidth]{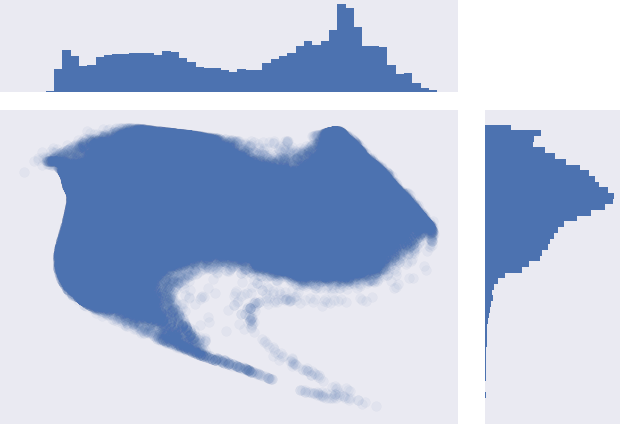}\\
\caption{Two dimensional example of the problem of unmixing natural images. One can easily spot that \wica{} has the smallest amount of artifacts remained after retrieving the signals. All of the scatter plots were normalized and are presented in the same scale. It is valuable to also look at the attached marginal histograms, where some of the similarities between the original signal and its retrieved counterpart may be observed.
% \jacek{dodac opis dlaczego warto patrzec na histogramy, ze podobienstwo do oryginalu tez widac}
}\label{fig:2d-nonlinear_1}
\end{figure}

In order to empirically demonstrate this property, we artificially mixed a multidimensional grid using the mixing function from Section \ref{sec:mixing}. Next, we randomly swapped one of the mixed signals with the original signal from the base dataset. We compared this \textit{mixed-and-swapped} data to the source signals using $max\_corr$ and \mixname{}. The results over different mixing iterations are presented on Fig. \ref{spearmanVSmaxcorr}. One may observe that $max\_corr$ values are always above the \mixname{} ones, suggesting that $max\_corr$ measure prefers such a recovery more than \mixname{}. Naturally, in the case when all signals are far different from the true sources, values for $max\_corr$ and $\mixname{}$ are almost exactly the same (see Fig. \ref{fullMixing}).

In consequence, the $max\_corr$ measure can help to asses the maximum of informativeness from the retrieved signal. This can be desired in situations that favor well decomposition of few components at the cost of lower correlatedness of the remaining ones (which may happen, for instance, in denoising problems). In the case where approximately equal recovery of all the signals is requested, the \mixname{} measure would be a better choice.  %One needs to select the correct measure for a given task individually, based on the result that is preferred.

\section{Experiments} \label{sec:experiments}

% \ola{WAZNE!: Sprawdzic zeby wszedzie w ekspery,mentach byl jeden czas - podobno najelpiej zeby byl oto czas przeszly, czyli pokazalismy etc.} 

In this section we show several simulated experiments to validate the \wica{} algorithm empirically. Because there is no clear benchmark definition for the nonlinear ICA evaluation, we have selected most figurative and easily interpretable setup which we present in the following subsections. In addition, we performed the analysis of electroencephalographic (EEG) signal according to procedure presented in \Citep{icaEEG,Onton_2006}, to validate our method in more natural setting, that is, without artificially generated mixing and access to true source components.

% \subsection{Training Regime for Synthetic Datasets.}

% In order to compare the \wica{} algorithm to other nonlinear ICA approaches we evaluated their performance in the case of separation of artificially mixed images from the Berkeley Segmentation Dataset\footnote{\url{https://www2.eecs.berkeley.edu/Research/Projects/CS/vision/bsds/BSDS300/html/dataset/images.html}}. We uniformly sampled $d$ flattened images from the Berkeley Segmentation Dataset to form the source components. We repeated this procedure $5$ times, in order to use more than one set of images per each dimension $d \in \{2,4,6,8,10\}$. The observations were obtained by using the function described in Section \ref{sec:mixing}, applied iterativly $i \in \{10, 20, 30, 40, 50\}$ times. 

% We compared the proposed $\wica{}$ method with dCor~\Citep{dCor},  ANICA~\Citep{bengio2013representation}, PNLMISEP~\Citep{pnlmisep}, and linear FastICA. We evaluated every method $10$ times for each pair of dimension and mix iterations $(d,i)_j, j \in \{1, \dots 5\}$. The experiment results measured by $max\_corr$ and $\mixname{}$ are discussed in the following subsections. In addition, we present the numeric details for all used methods and measures in Table \ref{fig:max_corr-50}. 

\subsection{Qualitative results} \label{sec:imageseparation}
% \andrzej{nie mam pomysłu na tytuł subsekcji}
% \ola{One of the most figurative application of ICA is the separation of images. The experiment environment in such a setting is constructed by applying some mixing function to the independent source signals (images), which are then passed to the tested ICA model. } 
% One of the most figurative application of ICA is the separation of images. The experiment environment in such setting is constructed by applying some mixing function to the independent source signal, which are then passed to the tested ICA~model.
%To create test environment in such setting, one need to apply some mixing to source signals and feed such preprepared mixture to the tested model. In the results, retrived signals should be closely matching original samples took into the mix.

We start from the simulated example of the ICA application in the case of images separation problem. We use this regime because the results can be understand with the naked eye of a reader.

% \ola{Czy na pewno chcemy pisac ze moga byc comapred przez max\_corr przy experyemncie w ktorym tego nie robimy? moze lepiej ostatnie zdanie z tego paragrafu powiedziec dopiero w sekcji 6.2} 
To construct this experiment one needs to apply some artificial mixing function (i.e. linear transformation or mixing function from Section \ref{sec:mixing}) on the independent source signals. Such mixture is then passed to the ICA model in question to perform the unmixing task. 

% Results of different ICA algorithms can be then validated using some dependency measure between the source signals and the recovered by an algorithm components.

% \ola{SPRAWDZIC: czy Berkely Segemntatio nDataset prosi o cytyowanie jakiejs konkrentej pracy?}
In order to compare the \wica{} algorithm to other nonlinear ICA approaches we evaluated the models performance in the case of separation of artificially mixed images. As an initial setup for this blind source separation task, we randomly sampled two flattened images from the Berkeley Segmentation Dataset \Citep{MartinFTM01}\footnote[4]{available at \url{https://www2.eecs.berkeley.edu/Research/Projects/CS/vision/bsds/BSDS300/html/dataset/images.html}} and mixed them using the function defined in Section \ref{sec:mixing}. We compared the proposed our method with dCor~\Citep{spurek2020non}, PNLMISEP~\Citep{pnlmisep}, ANICA~\Citep{bengio2013representation} and linear FastICA. Results of this toy example are presented on Fig.~\ref{fig:2d-nonlinear_1}.

Besides the retrieved images and their scatter plots, we also demonstrated projection of marginal densities. The desired goal is to achieve similar images and marginal densities as in the source (original) pictures. 

One can easily spot that FastICA and dCor seem to only rotate the mixed signals. The ANICA, on the other hand, transformed the observations to a high extent, but the recovered signals are visually worse than the original pictures. Similarly to previous algorithm, PNLMISEP and \wica{} also performed some nontrivial shift on the marginal densities, but in this case the retrieved densities resemble the original ones more naturally. 

This experiment was fully qualitative and the outcome is subject to one’s individual perception. We demonstrated the images purely as a visualization of the different ICA models performance in simple nonlinear setup. We report quantitative results in the next subsection. 

\begin{figure}[!htb]
\centering
\includegraphics[width=0.9\linewidth]{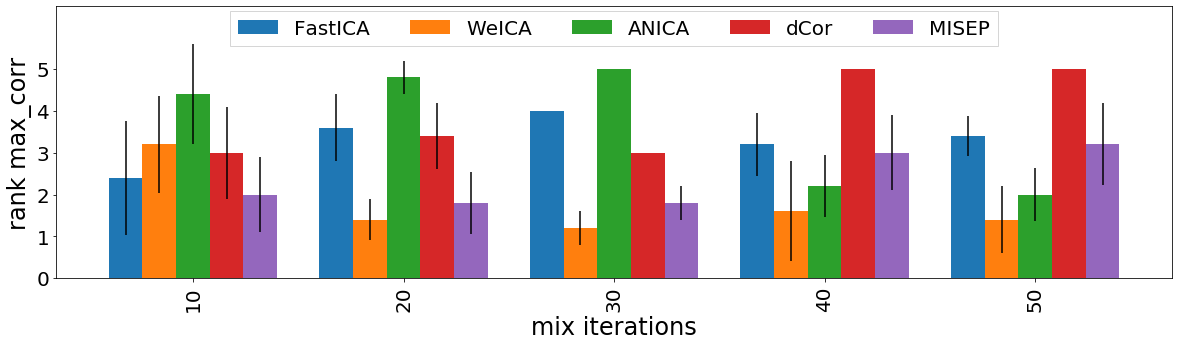}
\\ 
\includegraphics[width=0.9\linewidth]{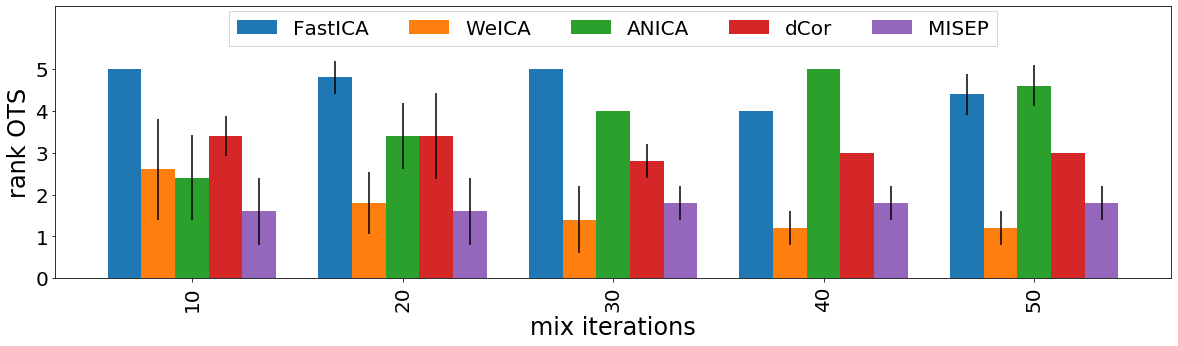}
\caption{The mean rank results for different mixes measured by $max\_corr$ (top) and \mixname{} (bottom). The lower the better. }\label{fig:ranks}
\end{figure}

\begin{figure*}[!h]
\centering
\subfigure[Original EEG signals]{
\includegraphics[width=0.28\linewidth]{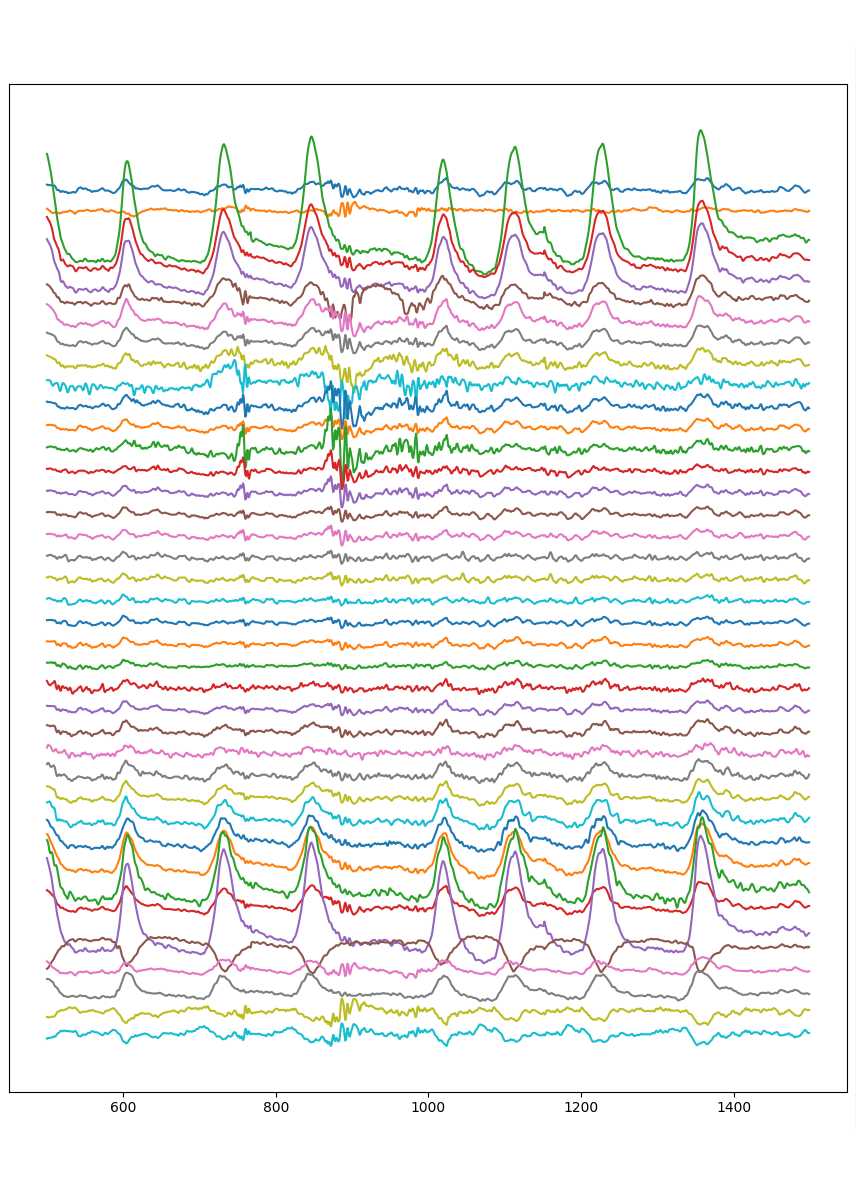}}
\subfigure[Retrieved by \wica{}]{
\includegraphics[width=0.28\linewidth]{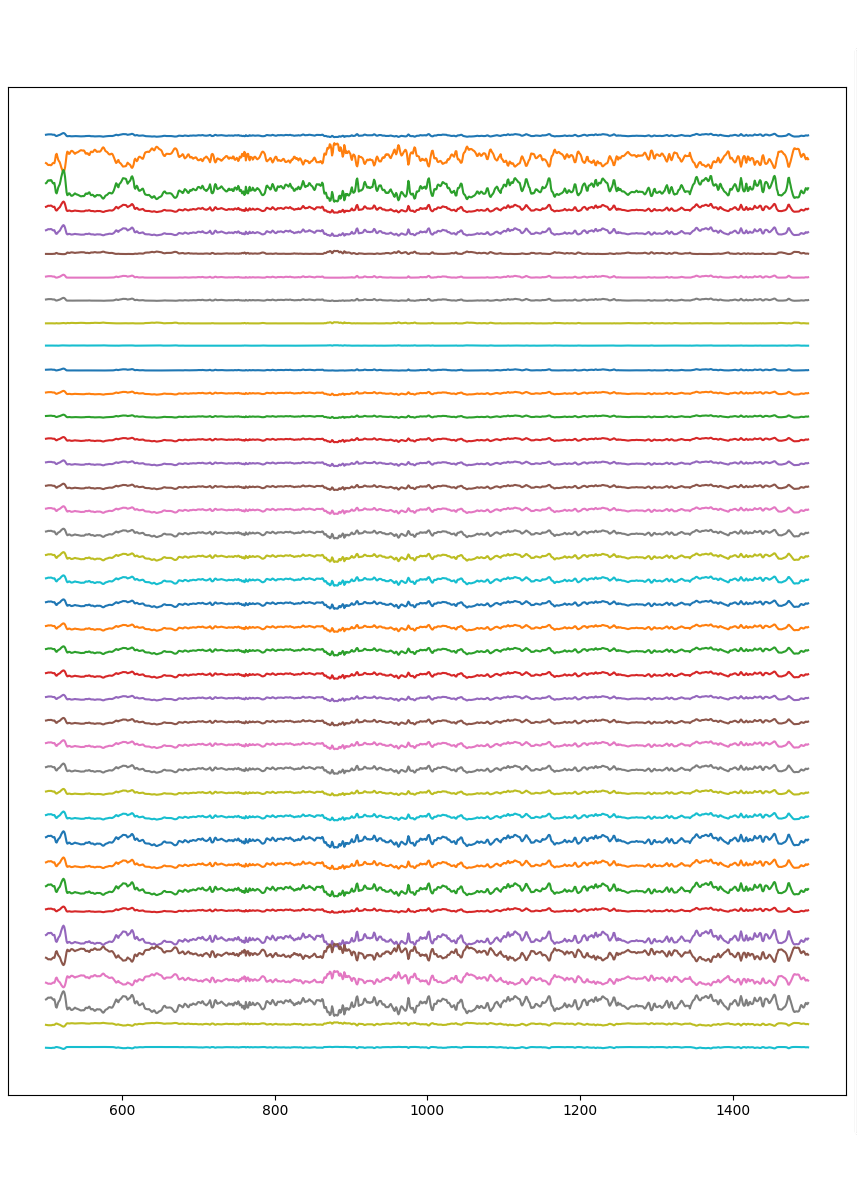}}
\subfigure[Retrieved by FastICA]{
\includegraphics[width=0.328\linewidth]{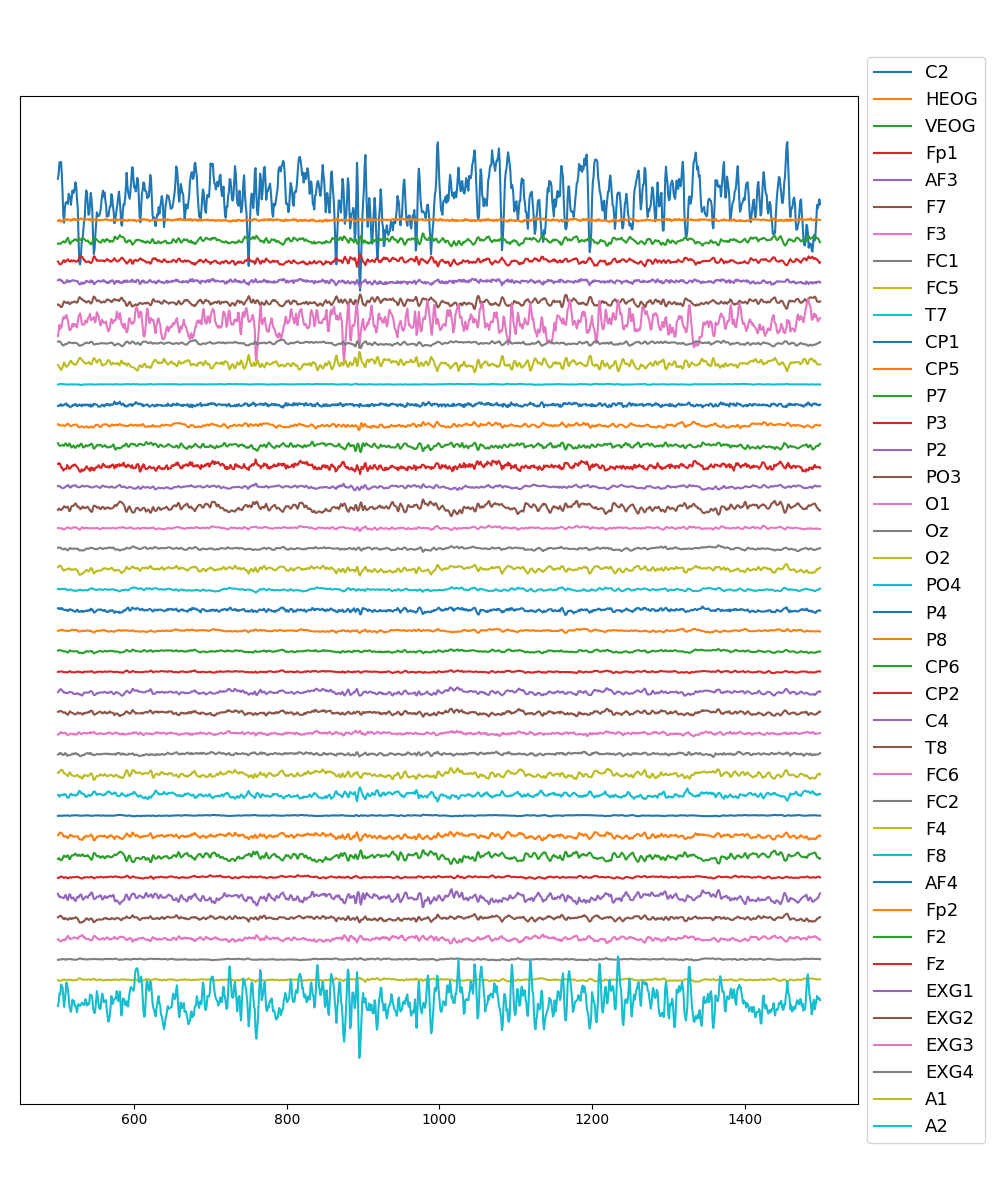}}

\caption{Results of analysis done on the EEG signals. After the deletion of a suspicious signals selected by an expert from the decomposition, one can easily spot that the reconstructed components are more homogeneous, and do not have as much artifacts as the original EEG data. In both methods the same amount of signals was~cleared. The results are satisfying in either of the cases. Additionally \wica{} persist scale of the retrieved signals, which is helpful property in further cleansing of the EEG data.}\label{fig:eeg_retived}
\end{figure*}

\subsection{Quantitative results}
% \andrzej{nie mam pomysłu na tytuł subsekcji}

From the preliminary results reported in previous subsection, we moved to a more complex scenario in which we quantitative evaluated the ICA methods in a higher dimensional setup.

We uniformly sampled $d$ flattened images from the Berkeley Segmentation Dataset \Citep{MartinFTM01} to form the source components. We used five different source dimensions $d \in \{2,4,6,8,10\}$. The observations were then obtained by using the function described in Section \ref{sec:mixing}, applied iteratively $i \in \{10, 20, 30, 40, 50\}$ times. For each dimension $d$ we randomly picked $5$ different sets of source images. Every method was evaluated $10$ times on each set of sources, dimensions and mixes. 
%The setup of this experiment is as follows. One take some value $d \in \{2,4,6,8,10\}$, which is a dimension of base signal (in our case number of base images taken into mixing procedure). This $d$-dimensional signal is then nonlinearly mixed using the function described in Section \ref{sec:mixing}. Produced mixture becomes an input for tested ICA models. All recovered signals, from each of the ICA method, are then evaluated using $max\_corr$ and \mixname{} measures. 

%It is worth to mention, that we wanted to test algorithms in various situations where the nonlinearity of the setting in consideration is shallow or very complex. To achieve various complexity of the test environment, we applied our mixing $n$ times iteratively. We~choose~$n$~from the set of $\{10,20,30,40,50\}$ values.

% \ref{sec:imageseparation}.

%  \ola{Tu nizej trzeba napisac na jakim datasecie!!!!! }
We fit each nonlinear algorithm using the grid search over the~learning rate. For the auto-encoder based models we also performed a grid search over the scaling of the independence measure. Adjustment of these hyper-parameters was done on randomly sampled observations from the set of all obtained mixtures. %dimension and mixing complexity. 
Examples used to tune the architectures, were then excluded from the dataset on which we performed the actual evaluation. It is worth to mention that we had to fix batch size to $256$, because any bigger value caused instabilities in the ANICA results. To be fair in comparisons, we set the same neural net architecture for \wica{}, ANICA and dCor. Both the encoder and the decoder were composed of~3~hidden layers with $128$ neurons each. In the case of MISEP we used the PNL version from \Citep{pnlmisep}. The outcomes from each method were measured both by $max\_corr$ and $\mixname{}$ against the true source components.

\begin{figure*}[!h]
\centering

\includegraphics[width=0.49\linewidth]{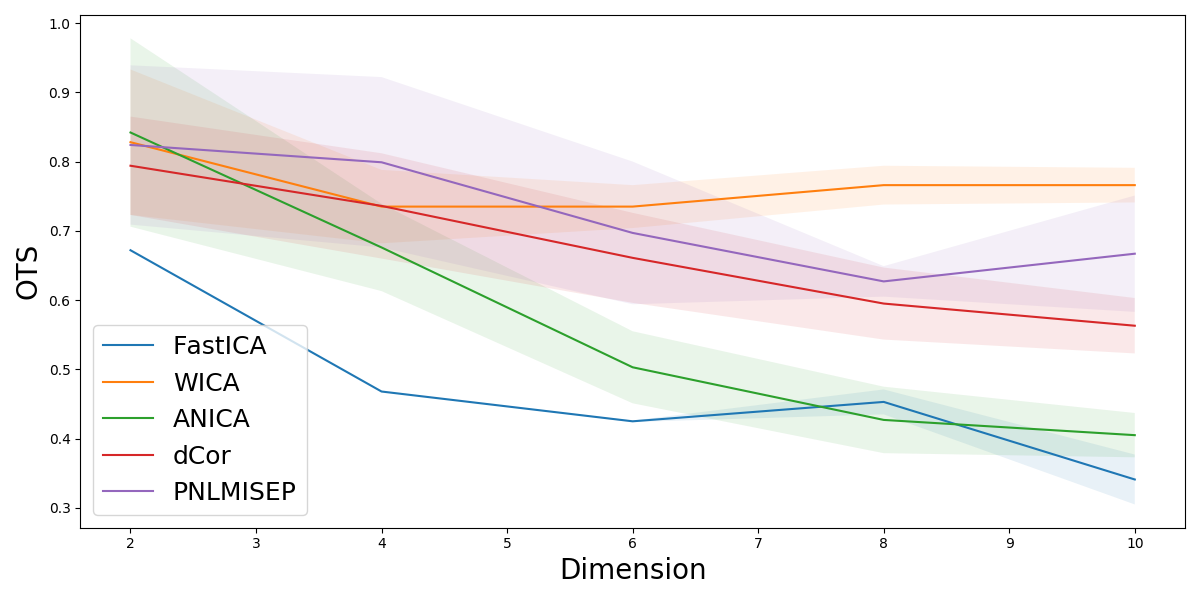}
\hfill
\includegraphics[width=0.49\linewidth]{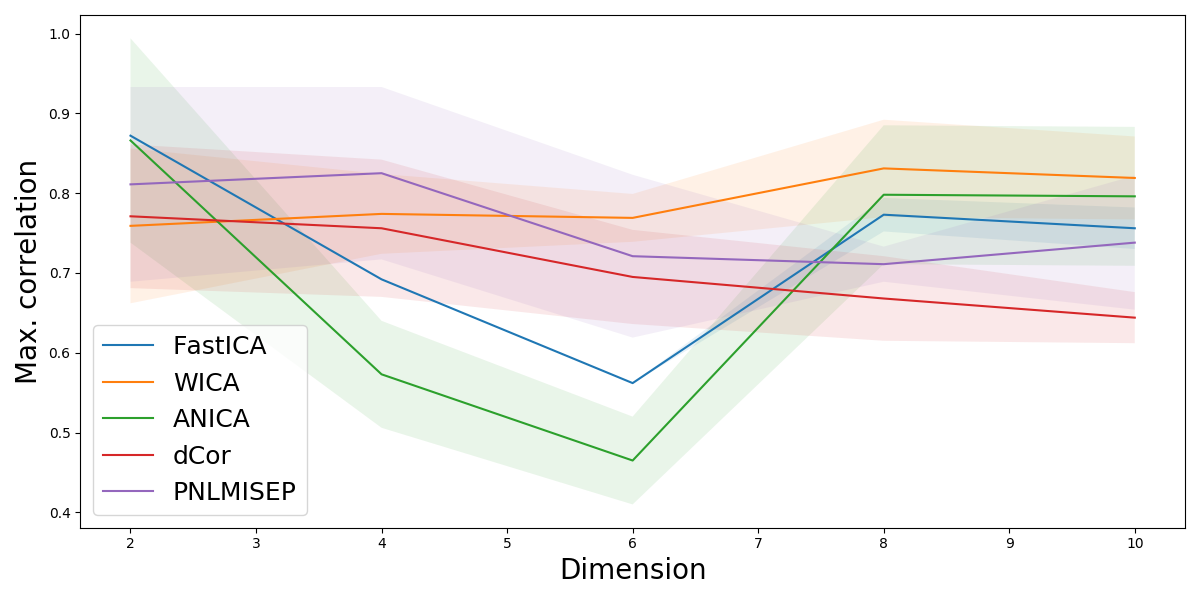}
% }
\caption{Comparison between standard ICA methods (PNLMISEP, dCor, ANICA, FastICA) and our approach 
by using \mixname{} (left) and $max\_corr$ (right) measures in the setup where $50$ mixing iterations were performed. In the experiment we train five models and present the mean and standard deviation of each of the used measures (the higher the better). One can observe that \wica{} consistently obtains good results for all of the dimensions and outperforms the other methods in higher dimensions. Moreover, it has the lowest standard deviation across all the nonlinear algorithms. More numerical results of the experiment are presented in Table \ref{fig:table}.}\label{fig:spearman_1}
\end{figure*}

\paragraph{Performance across different dimensions.} We plotted the results of this experiment on $50$ mixes\footnote[5]{We considered the setting with $50$ mixing iterations as the hardest one.} with respect to the data dimension $d$ in Fig.~ \ref{fig:spearman_1}. The outcomes demonstrated that the \wica{} method outperformed any other nonlinear algorithm in the proposed task by achieving high and stable results regardless of the considered data dimension. In~the~case of the results stability, \wica{}~losses only to the linear method -- FastICA -- which, unfortunately, cannot satisfactorily factorize nonlinear data. This experiment demonstrated that \wica{} is a strong competitor to other models in a fully unsupervised environment for nonlinear~ICA.

% \ola{Cos w tym stylu jest ok?}. 
It is also worth to mention the difference between the results measured by \mixname{} and $max\_corr$ for the ANICA and FastICA models applied in high dimension. We hypothesize that this may indicate that those algorithms were able to retrieve very well only small subset of the components, while the remaining variables were still highly mixed, leading to a similar effect as the one described in Section \ref{sec:OTS}.

\paragraph{Performance across different mixes.} For every model we evaluated the mean \mixname{} and $max\_corr$ score on a given dimension $d$ and number of mixing iterations $i$. Then, for each pair $(d,i)$ we ranked the tested models based on their performance. We report the mean rank of models for each mixing iteration $i$ in Fig. \ref{fig:ranks} (the lower the better). 

One may observed that for tasks relatively similar to the linear case, where number of mixes is equal to $10$, the PNLMISEP method performs the best both on $max\_corr$ and \mixname{}. However, as the number of mixes increases, the \wica{} algorithm usually outperforms all the other methods in both measures, achieving the lowest mean rank.
As a complement to the above discussion we also provide the complete numerical results for all mixtures on all tested dimensions in Table~\ref{fig:table}. 

\begin{table*}[!h]

\centering
{\small
% \usepackage{multirow}
\begin{tabular}{l|c|c|c|c|c|c|c|}
\toprule
\multicolumn{1}{|l|}{\textbf{Measure}} & \multicolumn{1}{c|}{\textbf{Dim}} & \multicolumn{1}{c|}{\textbf{Mixes}} & \multicolumn{1}{c|}{\textbf{ \wica{} }} & \multicolumn{1}{c|}{\textbf{FastICA}} & \multicolumn{1}{c|}{\textbf{ANICA}} & \multicolumn{1}{c|}{\textbf{dCor}} & \textbf{PNLMISEP}                      \\ \midrule
\multirow{25}{*}{$max\_corr$}   & 2                                                  &                                                      & 0.771$\pm$0.013                                                           & \textbf{0.965$\pm$0.001}              & 0.631$\pm$0.112                                      & 0.901$\pm$0.058                                     & 0.942$\pm$0.045                                         \\
                              & 4                                                  &                                                      & \textbf{0.910$\pm$0.065}                                 & 0.710$\pm$0.000                                        & 0.588$\pm$0.063                                      & 0.552$\pm$0.278                                     & 0.645$\pm$0.383                                         \\
                              & 6                                                  & 10                                                   & \textbf{0.821$\pm$0.041}                                 & 0.578$\pm$0.000                                        & 0.505$\pm$0.062                                      & 0.696$\pm$0.059                                     & 0.808$\pm$0.063                                         \\
                              & 8                                                  &                                                      & \textbf{0.814$\pm$0.058}                                 & 0.759$\pm$0.046                                        & 0.769$\pm$0.065                                      & 0.658$\pm$0.044                                     & 0.812$\pm$0.085                                         \\
                              & 10                                                 &                                                      & 0.812$\pm$0.049                                                           & 0.770$\pm$0.058                                        & \textbf{0.837$\pm$0.042}            & 0.658$\pm$0.041                                     & 0.820$\pm$0.077                                         \\ \cmidrule{2-8}
                              & 2                                                  &                                                      & 0.870$\pm$0.088                                                           & 0.827$\pm$0.000                                        & 0.817$\pm$0.080                                      & \textbf{0.883$\pm$0.060}           & 0.853$\pm$0.156                                         \\
                              & 4                                                  &                                                      & \textbf{0.957$\pm$0.059}                                 & 0.751$\pm$0.000                                        & 0.559$\pm$0.059                                      & 0.756$\pm$0.065                                     & 0.833$\pm$0.069                                         \\
                              & 6                                                  & 20                                                   & \textbf{0.795$\pm$0.033}                                 & 0.574$\pm$0.015                                        & 0.480$\pm$0.053                                      & 0.696$\pm$0.052                                     & \textbf{0.795$\pm$0.056}               \\
                              & 8                                                  &                                                      & \textbf{0.844$\pm$0.055}                                 & 0.770$\pm$0.013                                        & 0.803$\pm$0.085                                      & 0.703$\pm$0.060                                     & 0.813$\pm$0.056                                         \\
                              & 10                                                 &                                                      & \textbf{ 0.858$\pm$0.065}                                & 0.743$\pm$0.010                                        & 0.751$\pm$0.056                                      & 0.634$\pm$0.051                                     & 0.688$\pm$0.058                                         \\ \cmidrule{2-8}
                              & 2                                                  &                                                      & 0.925$\pm$0.100                                                           & 0.819$\pm$0.001                                        & 0.702$\pm$0.100                                      & 0.824$\pm$0.067                                     & \textbf{0.939$\pm$0.037}               \\
                              & 4                                                  &                                                      & 0.820$\pm$0.051                                                           & 0.673$\pm$0.002                                        & 0.571$\pm$0.086                                      & 0.788$\pm$0.070                                     & \textbf{0.898$\pm$0.070}               \\
                              & 6                                                  & 30                                                   & \textbf{0.887$\pm$0.036}                                 & 0.572$\pm$0.000                                        & 0.521$\pm$0.058                                      & 0.687$\pm$0.075                                     & 0.752$\pm$0.052                                         \\
                              & 8                                                  &                                                      & 0.746$\pm$0.050                                                           & 0.800$\pm$0.005                                        & \textbf{0.827$\pm$0.055}            & 0.639$\pm$0.037                                     & 0.772$\pm$0.104                                         \\
                              & 10                                                 &                                                      & \textbf{0.835$\pm$0.052}                                 & 0.751$\pm$0.010                                        & 0.814$\pm$0.062                                      & 0.675$\pm$0.033                                     & 0.740$\pm$0.021                                         \\ \cmidrule{2-8}
                              & 2                                                  &                                                      & 0.862$\pm$0.057                                                           & 0.882$\pm$0.020                                        & 0.746$\pm$0.093                                      & 0.852$\pm$0.079                                     & \textbf{0.931$\pm$0.043}               \\
                              & 4                                                  &                                                      & \textbf{0.847$\pm$0.048}                                 & 0.681$\pm$0.003                                        & 0.580$\pm$0.071                                      & 0.761$\pm$0.101                                     & 0.836$\pm$0.094                                         \\
                              & 6                                                  & 40                                                   & 0.701$\pm$0.039                                                           & 0.585$\pm$0.010                                        & 0.468$\pm$0.059                                      & 0.696$\pm$0.039                                     & \textbf{0.815$\pm$0.064}               \\
                              & 8                                                  &                                                      & \textbf{0.861$\pm$0.055}                                 & 0.781$\pm$0.001                                        & 0.792$\pm$0.079                                      & 0.658$\pm$0.054                                     & 0.734$\pm$0.108                                         \\
                              & 10                                                 &                                                      & \textbf{0.859$\pm$0.056}                                 & 0.746$\pm$0.001                                        & 0.749$\pm$0.112                                      & 0.642$\pm$0.030                                     & 0.802$\pm$0.049                                         \\ \cmidrule{2-8}
                              & 2                                                  &                                                      & 0.759$\pm$0.097                                                           & \textbf{0.872$\pm$0.000}              & 0.866$\pm$0.128                                      & 0.771$\pm$0.090                                     & 0.811$\pm$0.122                                         \\
                              & 4                                                  &                                                      & 0.774$\pm$0.050                                                           & 0.692$\pm$0.000                                        & 0.573$\pm$0.067                                      & 0.756$\pm$0.086                                     & \textbf{0.825$\pm$0.108}               \\
                              & 6                                                  & 50                                                   & \textbf{0.769$\pm$0.030}                                 & 0.562$\pm$0.000                                        & 0.465$\pm$0.055                                      & 0.695$\pm$0.059                                     & 0.721$\pm$0.102                                         \\
                              & 8                                                  &                                                      & \textbf{0.831$\pm$0.061}                                 & 0.773$\pm$0.021                                        & 0.798$\pm$0.087                                      & 0.668$\pm$0.053                                     & 0.711$\pm$0.022                                         \\
                              & 10                                                 &                                                      & \textbf{0.819$\pm$0.052}                                 & 0.756$\pm$0.026                                        & 0.796$\pm$0.087                                      & 0.644$\pm$0.032                                     & 0.738$\pm$0.084                                         \\ \midrule

\multirow{25}{*}{\mixname{}}         & 2                                                  &                                                      & 0.798$\pm$0.048                                                           & 0.652$\pm$0.001                                        & 0.938$\pm$0.088                                      & 0.899$\pm$0.076                                     & \textbf{0.948$\pm$0.041}               \\
                              & 4                                                  &                                                      & \textbf{0.890$\pm$0.065}                                 & 0.582$\pm$0.000                                        & 0.784$\pm$0.062                                      & 0.554$\pm$0.264                                     & 0.652$\pm$0.360                                         \\
                              & 6                                                  & 10                                                   & \textbf{0.807$\pm$0.043}                                 & 0.419$\pm$0.000                                        & 0.571$\pm$0.056                                      & 0.666$\pm$0.064                                     & 0.779$\pm$0.046                                         \\
                              & 8                                                  &                                                      & \textbf{0.784$\pm$0.025}                                 & 0.457$\pm$0.058                                        & 0.431$\pm$0.054                                      & 0.594$\pm$0.051                                     & 0.769$\pm$0.097                                         \\
                              & 10                                                 &                                                      & 0.742$\pm$0.030                                                           & 0.405$\pm$0.077                                        & 0.405$\pm$0.032                                      & 0.556$\pm$0.041                                     & \textbf{0.758$\pm$0.103}               \\ \cmidrule{2-8}
                              & 2                                                  &                                                      & \textbf{0.884$\pm$0.089}                                 & 0.674$\pm$0.000                                        & 0.820$\pm$0.088                                      & 0.850$\pm$0.105                                     & 0.864$\pm$0.099                                         \\
                              & 4                                                  &                                                      & \textbf{0.945$\pm$0.064}                                 & 0.561$\pm$0.000                                        & 0.760$\pm$0.062                                      & 0.722$\pm$0.061                                     & 0.808$\pm$0.062                                         \\
                              & 6                                                  & 20                                                   & \textbf{0.776$\pm$0.031}                                 & 0.456$\pm$0.015                                        & 0.500$\pm$0.053                                      & 0.650$\pm$0.053                                     & 0.764$\pm$0.038                                         \\
                              & 8                                                  &                                                      & \textbf{0.797$\pm$0.029}                                 & 0.442$\pm$0.014                                        & 0.432$\pm$0.064                                      & 0.623$\pm$0.071                                     & 0.767$\pm$0.053                                         \\
                              & 10                                                 &                                                      & \textbf{0.790$\pm$0.029}                                 & 0.406$\pm$0.034                                        & 0.404$\pm$0.027                                      & 0.551$\pm$0.041                                     & 0.625$\pm$0.048                                         \\ \cmidrule{2-8}
                              & 2                                                  &                                                      & 0.797$\pm$0.105                                                           & 0.655$\pm$0.000                                        & 0.815$\pm$0.102                                      & 0.808$\pm$0.089                                     & \textbf{0.903$\pm$0.051}               \\
                              & 4                                                  &                                                      & 0.805$\pm$0.050                                                           & 0.468$\pm$0.001                                        & 0.661$\pm$0.074                                      & 0.760$\pm$0.074                                     & \textbf{0.884$\pm$0.065}               \\
                              & 6                                                  & 30                                                   & \textbf{0.865$\pm$0.037}                                 & 0.486$\pm$0.000                                        & 0.497$\pm$0.061                                      & 0.653$\pm$0.066                                     & 0.724$\pm$0.034                                         \\
                              & 8                                                  &                                                      & 0.702$\pm$0.033                                                           & 0.491$\pm$0.006                                        & 0.432$\pm$0.043                                      & 0.559$\pm$0.054                                     & \textbf{0.730$\pm$0.086}               \\
                              & 10                                                 &                                                      & \textbf{0.782$\pm$0.027}                                 & 0.440$\pm$0.021                                        & 0.405$\pm$0.034                                      & 0.560$\pm$0.063                                     & 0.657$\pm$0.043                                         \\ \cmidrule{2-8}
                              & 2                                                  &                                                      & 0.869$\pm$0.053                                                           & 0.668$\pm$0.001                                        & 0.868$\pm$0.087                                      & 0.852$\pm$0.087                                     & \textbf{0.938$\pm$0.032}               \\
                              & 4                                                  &                                                      & 0.781$\pm$0.049                                                           & 0.455$\pm$0.002                                        & 0.667$\pm$0.060                                      & 0.729$\pm$0.102                                     & \textbf{0.822$\pm$0.092}               \\
                              & 6                                                  & 40                                                   & 0.636$\pm$0.039                                                           & 0.405$\pm$0.008                                        & 0.516$\pm$0.059                                      & 0.649$\pm$0.044                                     & \textbf{0.772$\pm$0.071}               \\
                              & 8                                                  &                                                      & \textbf{0.729$\pm$0.029}                                 & 0.482$\pm$0.001                                        & 0.428$\pm$0.046                                      & 0.572$\pm$0.039                                     & 0.663$\pm$0.083                                         \\
                              & 10                                                 &                                                      & \textbf{0.820$\pm$0.025}                                 & 0.345$\pm$0.004                                        & 0.397$\pm$0.035                                      & 0.569$\pm$0.046                                     & 0.752$\pm$0.073                                         \\ \cmidrule{2-8}
                              & 2                                                  &                                                      & 0.828$\pm$0.105                                                           & 0.672$\pm$0.000                                        & \textbf{0.842$\pm$0.136}            & 0.794$\pm$0.071                                     & 0.824$\pm$0.115                                         \\
                              & 4                                                  &                                                      & 0.735$\pm$0.053                                                           & 0.468$\pm$0.000                                        & 0.676$\pm$0.063                                      & 0.736$\pm$0.076                                     & \textbf{0.799$\pm$0.123}               \\
                              & 6                                                  & 50                                                   & \textbf{0.735$\pm$0.031}                                 & 0.425$\pm$0.001                                        & 0.503$\pm$0.052                                      & 0.661$\pm$0.065                                     & 0.697$\pm$0.103                                         \\
                              & 8                                                  &                                                      & \textbf{0.766$\pm$0.028}                                 & 0.453$\pm$0.018                                        & 0.427$\pm$0.048                                      & 0.595$\pm$0.052                                     & 0.627$\pm$0.022                                         \\
                              & 10                                                 &                                                      & \textbf{0.766$\pm$0.025}                                 & 0.341$\pm$0.036                                        & 0.405$\pm$0.032                                      & 0.563$\pm$0.040                                     & 0.667$\pm$0.084    \\                       \bottomrule             
\end{tabular}
}

\caption{Comparison between nonlinear ICA methods (PNLMISEP, dCor, ANICA, WICA) and the classical linear ICA approach (FastICA) on images separation problem (with different dimensions) by using $max\_corr$ and \mixname{} measures. In the experiment we tuned and trained four models (excluding FastICA, which~is~a~linear~model) and present mean and standard deviation in the tabular form.}
\label{fig:table}
\end{table*}

\subsection{Decomposing EEG data}

Finally we want to show usability of the \wica{} method on real life data. An example of a task that can be tackle by the ICA algorithms is electroencephalogram (EEG) decomposition.

An EEG signal is~a~test used to evaluate the electrical activity in the brain. The brain cells communicate via electrical impulses and are active all the time. In~the~original scalp channel data, each row of the data recording matrix represents the time course of summed voltage differences between source projections to one data channel and one or more reference channels. We followed a common experiment framework proposed in \Citep{icaEEG, Onton_2006}, to detect artifacts in unmixed signals representation which can suggest a blinks or an eye movement during the test.

The setup for this decomposition is different than in previous sections. An original EEG mixture took for this experiment, consisted of~$40$~scalp electrode signals. Those signals were selected as an input for the \wica{} model. Retrieved data were analysed by an expert, who selected signs of a blinking on recovered components. Manually selected subset of suspicious components, were then nullified. Unmixed signal with masked (by nullification) components were then feed back to the decoder which came from the training of the \wica{} model. 

As a researcher we are not aware how deeply EEG signals are mixed or dependent. The crucial functionality that ICA serves in this setting is normalizing and cleansing of the dataset.  From that point, time series produced from recovered signals have to be analysed by an expert. In this experiment we want to prove that high dimension of the input data and the unknown entanglement of the components is not a limitation for the \wica{}. Visual results of this experiment are presented on Fig \ref{fig:eeg_retived}. For a comparison we used results from other standard ICA algorithm used for this kind of a task -- linear FastICA. The details of "remixing" process for this method are descirbed in \Citep{icaEEG}.
This experiment showed that \wica{} is able to handle multidimensional data highly above the volume tested for other nonlinear models. 
Moreover, results our method for this task works well enough to be used as a preliminary step of cleaning the data.

%%%%%%%%%%%%%%%%%%%%%%%%%%%%%%%

\section{Conclusion}

In this paper we presented a new approach to the nonlinear ICA task. 

In addition to the investigation of \wica{} method, which proves to be matching the results of all other tested nonlinear algorithms, we proposed a new mixing function for validating nonlinear tasks in a structurized manner. Our mixing scales to higher dimensions and is easily invertible.

Lastly, we defined \mixname{}, a measure that can catch nonlinear dependence and is easy to compute. The \mixname{} measure and the proposed mixing have the potential to become benchmarking tools for all future work in this field. 

\section{Acknowledgements}

The work of P. Spurek was supported by the National Centre of Science (Poland) Grant No. 2019/33/B/ST6/00894. The work of J. Tabor was supported by the National Centre of Science (Poland) Grant No. 2017/25/B/ST6/01271. 
A. Nowak carried out this work within the research project "Bio-inspired artificial neural networks" (grant no. POIR.04.04.00-00-14DE/18-00) within the Team-Net program of the Foundation for Polish Science co-financed by the European Union under the European Regional Development Fund.

\end{document}